\documentclass[11pt]{article}

\usepackage{fullpage}

\usepackage{times}
\usepackage{amsthm,amsfonts,amsmath,amssymb,color,float,graphicx,verbatim}
\usepackage{subfigure} 
\usepackage{natbib}
\usepackage{algorithm,algorithmic}

\usepackage{mathtools}
\usepackage{bbm}
\usepackage{csquotes}

\newtheorem{theorem}{Theorem}
\newtheorem{proposition}{Proposition}
\newtheorem{claim}{Claim}
\newtheorem{lemma}{Lemma}

\newtheorem{definition}{Definition}
\newtheorem{assumption}{Assumption}

\newcommand{\pr}[1]{\mathbb{P}\left[ #1 \right]}

\newcommand{\ceil}[1]{\lceil #1 \rceil}
\newcommand{\abs}[1]{\left| #1 \right|}
\newcommand{\one}{\mathbbm{1}}
\newcommand{\reals}{\mathbb{R}}

\newcommand{\sign}{\mathrm{sign}}

\newcommand{\argmin}[1]{\underset{#1}{\mathrm{argmin}}}

\newcommand{\rank}{\text{rank}}
\newcommand{\relu}[1]{\left[ #1 \right]_+}
\newcommand{\coleq}{\coloneqq}
\newcommand{\set}[1]{\left\lbrace#1\right\rbrace}
\newcommand{\bas}{\text{Bas}}
\newcommand{\p}[1]{\left( #1 \right)}
\newcommand{\pcc}[1]{\left[ #1 \right]}
\newcommand{\poc}[1]{\left( #1 \right]}
\newcommand{\pco}[1]{\left[ #1 \right)}
\newcommand{\bbr}{\mathbb{R}}

\newcommand{\ba}{\mathbf{a}}

\newcommand{\bx}{\mathbf{x}}
\newcommand{\bw}{\mathbf{w}}

\newcommand{\bb}{\mathbf{b}}
\newcommand{\bu}{\mathbf{u}}
\newcommand{\bv}{\mathbf{v}}
\newcommand{\bz}{\mathbf{z}}
\newcommand{\bc}{\mathbf{c}}

\newcommand{\by}{\mathbf{y}}

\newcommand{\bp}{\mathbf{p}}
\newcommand{\bq}{\mathbf{q}}

\newcommand{\balpha}{\boldsymbol{\alpha}}

\newcommand{\Ocal}{\mathcal{O}}

\newcommand{\Wcal}{\mathcal{W}}

\newcommand{\norm}[1]{\left\|#1\right\|}
\newcommand{\inner}[1]{\left\langle#1\right\rangle}

\newcommand{\secref}[1]{Sec.~\ref{#1}}
\newcommand{\subsecref}[1]{Subsection~\ref{#1}}
\newcommand{\figref}[1]{Fig.~\ref{#1}}
\renewcommand{\eqref}[1]{Eq.~(\ref{#1})}
\newcommand{\lemref}[1]{Lemma~\ref{#1}}

\newcommand{\thmref}[1]{Thm.~\ref{#1}}
\newcommand{\propref}[1]{Proposition~\ref{#1}}
\newcommand{\appref}[1]{Appendix~\ref{#1}}

\title{On the Quality of the Initial Basin in Overspecified Neural Networks}

\author{Itay Safran\\Weizmann Institute of Science\\\texttt{itay.safran@weizmann.ac.il}\and
	Ohad Shamir
	\\Weizmann Institute of Science\\\texttt{ohad.shamir@weizmann.ac.il}
}
\date{}

\begin{document}

\maketitle

\begin{abstract}
	
	Deep learning, in the form of artificial neural networks, has achieved remarkable practical success in recent years, for a variety of difficult machine learning applications. However, a theoretical explanation for this remains a major open problem, since training neural networks involves optimizing a highly non-convex objective function, and is known to be computationally hard in the worst case. In this work, we study the \emph{geometric} structure of the associated non-convex objective function, in the context of ReLU networks and starting from a random initialization of the network parameters. We identify some conditions under which it becomes more favorable to optimization, in the sense of (i) High probability of initializing at a point from which there is a monotonically decreasing path to a global minimum; and (ii) High probability of initializing at a basin (suitably defined) with a small minimal objective value. A common theme in our results is that such properties are more likely to hold for larger (``overspecified'') networks, which accords with some recent empirical and theoretical observations.
\end{abstract}

\section{Introduction}

Deep learning (in the form of multi-layered artificial neural networks) has been tremendously successful in recent years, and advanced the state of the art across a range of difficult machine learning applications. Inspired by the structure of biological nervous systems, these predictors are usually composed of several layers of simple computational units (or neurons), parameterized by a set of weights, which can collectively express highly complex functions. Given a dataset of labeled examples, these networks are generally trained by minimizing the average of some loss function over the data, using a local search procedure such as stochastic gradient descent. 

Although the expressiveness and statistical performance of such networks is relatively well-understood, it is a major open problem to understand the computational tractability of training such networks. Although these networks are trained successfully in practice, most theoretical results are negative. For example, it is known that finding the weights that best fit a given training set, even for very small networks, is NP-hard \citep{blum1992training}. Even if we relax the problem by allowing improper learning or assuming the data is generated by a network, the problem remains worst-case hard (see e.g. \citep{livni2014computational} for a discussion of this and related results). This theory-practice gap is a prime motivation for our work.

In this paper, we study the \emph{geometric structure} of the objective function associated with training such networks, namely the average loss over the training data as a function of the network parameters. We focus on plain-vanilla, feedforward networks which use the simple and popular  ReLU activation function (see \secref{subsec:relu_nn} for precise definitions), and losses convex in the network's predictions, for example the squared loss and cross-entropy loss. The structure of the resulting objective function is poorly understood. Not surprisingly, it is complex, highly non-convex, and local search procedures are by no means guaranteed to converge to a global minimum. Moreover, it is known that even if the network is composed of a single neuron, the function may have exponentially many local minima \citep{auer1996exponentially}. Furthermore, as we discuss later in the paper, the construction can be done such that the vast majority of these local minima are sub-optimal. Nevertheless, our goal in this work is to understand whether, perhaps under some conditions, the function has some geometric properties which may make it more favorable to optimization.

Before continuing, we emphasize that our observations are purely geometric in nature, independent of any particular optimization procedure. Moreover, we make no claim that these properties necessarily imply that a practical local search procedure, such as stochastic gradient descent, will converge to a good solution (although proving such a result could be an interesting direction for future work). Nevertheless, the properties we consider do seem indicative of the difficulty of the optimization problem, and we hope that our results can serve as a basis for further progress on this challenging research direction.

A recurring theme in our results is that such favorable properties can be shown to occur as the network size grows \emph{larger}, perhaps larger than what would be needed to get good training error with unbounded computational power (hence the term \emph{overspecified} networks). At first, this may seem counter-intuitive, as larger networks have more parameters, and training them involves apparently more complex optimization in a higher-dimensional space. However, higher dimensions also means more potential directions of descent, so perhaps the gradient descent procedures used in practice are more unlikely to get stuck in poor local minima and plateaus. Although difficult to formalize, this intuition accords with several recent empirical and theoretical evidence, which indicates that larger networks may indeed be easier to train (see \citep{livni2014computational} as well as \citep{choromanska2014loss,dauphin2014identifying,bach2014breaking}).

In the first part of our work (\secref{sec:path}), we consider networks of arbitrary depth, where the weights are initialized at random using some standard initialization procedure. This corresponds to a random starting point in the parameter space. We then show that under some mild conditions on the loss function and the data set, as the network width increases, we are overwhelmingly likely to begin at a point from which there is a continuous, strictly monotonically decreasing path to a global minimum\footnote{To be precise, we prove a more general result, which implies a monotonic path to any objective value smaller than that of the initial point, as long as some mild conditions are met. See \thmref{thm:path} in \secref{sec:path} for a precise formulation.}. This means that although the objective function is non-convex, it is not ``wildly'' non-convex in the sense that the global minima are in  isolated valleys which cannot be reached by descent procedures starting from random initialization. In other words, ``crossing valleys'' is not strictly necessary to reach a good solution (although again, we give no guarantee that this will happen for a specific algorithm such as stochastic gradient descent). We note that this accords well with recent empirical observations \cite{goodfellow2014qualitatively}, according to which the objective value of networks trained in practice indeed tends to decrease monotonically, as we move from the initialization point to the end point attained by the optimization algorithm. We also note that although we focus on plain-vanilla feed-forward networks, our analysis is potentially applicable to more general architectures, such as convolutional networks.

In the second part of our work (\secref{sec:Neural_nets_depth_two}), we focus more specifically on two-layer networks with scalar-valued outputs. Although simpler than deeper networks, the associated optimization problem is still highly non-convex and exhibits similar worst-case computational difficulties. For such networks, we study a more fine-grained geometric property: We define a partition of the parameter space into convex regions (denoted here as basins), in each of which the objective function has a relatively simple, basin-like structure: Inside each such basin, every local minima of the objective function is global, all sublevel sets are connected, and in particular there is only a single connected set of minima, all global on that basin. We then consider the probability that a random initialization will land us at a basin with small minimal value. Specifically, we show that under various sets of conditions (such as low intrinsic data dimension, or a cluster structure), this event will occur with overwhelmingly high probability as the network size increases. As an interesting corollary, we show that the construction of \citep{auer1996exponentially}, in which a single neuron network is overwhelmingly likely to be initialized at a bad basin, is actually surprisingly brittle to overspecification: If we replace the single neuron with a two-layer network comprised of just $\Omega(\log(d))$ neurons ($d$ being the data dimension), and use the same dataset, then with overwhelming probability, we will initialize at a basin with a globally optimal minimal value. 

As before, we emphasize that these results are purely geometric, and do not imply that an actual gradient descent procedure will necessarily attain such good objective values. Nevertheless, we do consider a property such as high probability of initializing in a good basin as indicative of the optimization difficulty of the problem. 

We now turn to discuss some related work. Perhaps the result most similar to ours appears in  \citep{livni2014computational}, where it is shown that quite generally, if the number of neurons in the penultimate layer is larger than the data size, then global optima are ubiquitous, and ``most'' starting points will lead to a global optimum upon optimizing the weights of the last layer. Independently, \citep{haeffele2015global} also provided results of a similar flavor, where sufficiently large networks compared to the data size and dimension do not suffer from local minima issues. However, these results involve huge networks, which will almost invariably overfit, whereas the results in our paper generally apply to networks of more moderate size. Another relevant work is \citep{choromanska2014loss}, which also investigates the objective function of ReLU networks. That work differs from ours by assuming data sampled from a standard Gaussian distribution, and considering asymptotically large networks with a certain type of random connectivity. This allows the authors to use tools from the theory of spin-glass models, and obtain interesting results on the asymptotic distribution of the objective values associated with critical points. Other results along similar lines appear in \citep{dauphin2014identifying}. This is a worthy but rather different research direction than the one considered here, where we focus on theoretical investigation of non-asymptotic, finite-sized networks on fixed datasets, and consider different geometric properties of the objective function. Other works, such as \citep{arora2014provable,andoni2014learning,janzamin2015beating,zhang2015learning} and some of the results in \citep{livni2014computational}, study conditions under which certain types of neural networks can be efficiently learned. However, these either refer to networks quite different than standard ReLU networks, or focus on algorithms which avoid direct optimization of the objective function (often coupled with strong assumptions on the data distribution). In contrast, we focus on the geometry of the objective function, which is directly optimized by algorithms commonly used in practice. Finally, works such as
\citep{bengio2005convex,bach2014breaking} study ways to convexify (or at least simplify) the optimization problem by re-parameterizing and lifting it to a higher dimensional space. Again, this involves changing the objective function rather than studying its properties.

\section{Preliminaries and Notation}
We use bold-faced letters to denote vectors, and capital letters to generally denote matrices. Given a natural number $k$, we let $\pcc k$ be shorthand for $\set{1,\dots,k}$. 

\subsection{ReLU Neural Networks}
	\label{subsec:relu_nn}
	We begin by giving a formal definition of the type of  neural network studied in this work. A fully connected feedforward artificial neural network computes a function $\bbr^d\to\bbr^k$, and is composed of neurons connected according to a directed acyclic graph. Specifically, the neurons can be decomposed into layers, where the output of each neuron is connected to all neurons in the succeeding layer and them alone. We focus on ReLU networks, where each neuron computes a function of the form $\bx\mapsto\relu{\bw^{\top}\bx+b}$ where $\bw$ is a weight vector and $b$ is a bias term specific to that neuron, and $ \relu{\cdot} $ is the ReLU activation function $ \relu{z}=\max\set{0,z} $.
	
	For a vector $\bb=\p{b_1,\dots,b_n}$ and a matrix 
	\[
		W=\p{\begin{matrix}\dots & \bw_1 & \dots\\
		& \vdots\\
		\dots & \bw_n & \dots
		\end{matrix}},
	\]
	and letting $\relu{W\bx+\bb}$ be a shorthand for $\p{\relu{\bw_1^{\top}\bx+b_1},\ldots,\relu{\bw_n^{\top}\bx+b_n}}$,
	we can define a layer of $n$ neurons as
	\[
		\bx\mapsto\relu{W\bx+\bb}.
	\]
	Finally, by denoting the output of the $i^{\text{th}}$ layer as $O_i$, we can define a network of arbitrary depth recursively by
	\[
		O_{i+1}=\relu{W_{i+1}O_i+\bb_{i+1}},
	\]
	where $W_i,\bb_i$ represent the matrix of weights and bias of the $i^{\text{th}}$ layer, respectively. Following a standard convention for multi-layer networks, the final layer $h$ is a purely linear function with no bias, i.e.
	\begin{equation}
		\label{eq:deep_objective}
		O_h=W_h\cdot O_{h-1}.
	\end{equation}
	Define the \emph{depth} of the network as the number of layers $h$, and denote the number of neurons $n_i$ in the $i^{\text{th}}$ layer as the \emph{size} of the layer. We define the \emph{width} of a network as $\max_{i\in\pcc h}n_i$.
	
	We emphasize that in this paper, we focus on plain-vanilla networks, and in particular do not impose any constraints on the weights of each neuron (e.g. regularization, or having convolutional layers).
	
	We define $ \Wcal $ to be the set of all network weights, which can be viewed as one long vector (its size of course depends on the size of the network considered). We will refer to the Euclidean space containing $ \Wcal $ as the \emph{parameter space}.
	
	Define the output of the network $ N:\bbr^d\to\bbr^k $ over the set of weights $ \Wcal $ and an instance $ \bx\in\bbr^d $ by
	\[
		N\p{\Wcal}\p{\bx}.
	\]
	Note that depending on the dimension of $W_h$, the network's output can be either a scalar (e.g. for regression) or a vector (e.g. for the purpose of multiclass classification). An important property of the ReLU function, which we shall use later in the paper, is that it is positive-homogeneous: Namely, it satisfies for all $ c\ge0 $ and $ x\in\bbr $ that
	\[
		\relu{c\cdot x}=c\cdot\relu{x}.
	\]
	
\subsection{Objective Function}
	We use $S=\p{\bx_t,\by_t}_{t=1}^m$ to denote the data on which we train the network, where $\bx_t\in\bbr^d$ represents the $t^{\text{th}}$ training instance and $\by_t\in\bbr^k$ represents the corresponding target output, and where $m$ is used to denote the number of instances in the sample.
	
	Throughout this work, we consider a loss function $\ell\p{\by,\by^\prime}$, where the first argument is the prediction and the second argument is the target value. We assume $\ell$ is convex in its first argument (e.g. the squared loss or the cross-entropy loss).
	
	In its simplest form, training a neural network corresponds to finding a combination of weights which minimizes the average loss over the training data. More formally, we define the objective function as
	\[
		L_S\p{N\p{\Wcal}}=\frac{1}{m}\sum_{t=1}^m \ell\p{N\p{\Wcal}\p{\bx_t},\by_t}.
	\]
	We stress that even though $\ell$ is convex as a function of the network's prediction, $ L_S\p{N\p{\Wcal}} $ is generally non-convex as a function of the network's weights. Also, we note that occasionally when the architecture is clear from context, we omit $ N\p{\cdot} $ from the notation, and write simply $ L_S\p{\Wcal} $.

\subsection{Basins}
	In \secref{sec:Neural_nets_depth_two}, we will 
	we consider a partition of the parameter space into convex regions, in each of which the objective function $ L_S\p{\Wcal} $ has a relatively simple basin-like form, and study the quality of the basin in which we initialize. In particular, we define a \emph{basin} with respect to $ L_S\p{\Wcal} $ as a closed and convex subset of the parameter space, on which $ L_S\p{\Wcal} $ has connected sublevel sets, and where each local minimum is global. More formally, we have the following definition:
	
	\begin{definition}(Basin)
		\label{def:Basin}
		A closed and convex subset $ B $ of our parameter space is called a basin if the following conditions hold:
		\begin{itemize}
			\item
			$ B $ is connected, and for all $ \alpha\in \bbr $, the set $ B_\le\alpha = \set{\Wcal\in B : L_S\p{\Wcal} \le \alpha} $ is connected.
			
			\item
			If $\Wcal\in B$ is a local minimum of $L_S$ on $B$, then it is a global minimum of $L_S$ on $B$.
			
		\end{itemize}
	\end{definition}
	We define the basin value $ \bas\p{B} $ of a basin $ B $ as the minimal value\footnote{For simplicity, we will assume this minimal value is actually attained at some point in the parameter space. Otherwise, one can refer to an attainable value arbitrarily close to it.} attained:
	\[
		\bas\p{B}\coleq\min_{\Wcal\in B} L_S\p{\Wcal}.
	\]
	Similarly, for a point $ \Wcal $ in the interior of a basin $ B $ we define its basin value as the value of the basin to which it belongs:
	\[
		\bas\p{\Wcal}\coleq\bas\p{B}.
	\]
	In what follows, we consider basins with disjoint interiors, so the basin to which $ \Wcal $ belongs is always well-defined.
	
\subsection{Initialization Scheme}

	\label{sub:init_scheme}
	
	As was mentioned in the introduction, we consider in this work questions such as the nature of the basin we initialize from, under some random initialization of the network weights. Rather than assuming a specific distribution, we will consider a general class of distributions which satisfy some mild independence and symmetry assumptions:
	\begin{assumption}
		\label{assum:init_dist}
		The initialization distribution of the network weights satisfies the following:
		\begin{itemize}
			\item
			The weights of every neuron are initialized independently.
			\item
			The vector of each neuron's weights (including bias) is drawn from a spherically symmetric distribution supported on non-zero vectors.
		\end{itemize}
	\end{assumption}
	
	This assumption is satisfied by most standard initialization schemes: For example, initializing the weights of each neuron independently from some standard multivariate Gaussian, up to some arbitrary scaling, or initializing each neuron uniformly from an origin-centered sphere of arbitrary radius.	An important property of distributions satisfying Assumption \ref{assum:init_dist} is that the signs of the weights of each neuron, viewed as a vector in $\bbr^n$, is uniformly distributed on $\set{-1,+1}^n$.
	
\section{Networks of Any Depth: Path to Global Minima}
\label{sec:path}
In this section, we establish the existence of a continuous path in the parameter space of multilayer networks (of any depth), which is strictly monotonically decreasing in the objective value, and can reach an arbitrarily small objective value, including the global minimum. More specifically, we show in \thmref{thm:path} that if the loss is convex in the network's predictions, and there exists \emph{some} continuous path in the parameter space from the initial point $\Wcal^{\p{0}}$ to a point with smaller objective value $\Wcal^{\p{1}}$ (including possibly a global minimum, where the objective value along the path is not necessarily monotonic) which satisfies certain relatively mild conditions, then it is possible to find some other path from $\Wcal^{\p{0}}$ to a point as good as $\Wcal^{\p{1}}$, along which the objective value is strictly monotonically decreasing. 

For the theorem to hold, we need to assume our starting point has a sufficiently large objective value. In  \propref{prop:crosssquare} and \propref{prop:highinit}, we prove that this will indeed occur with random initialization, with overwhelming probability. A different way to interpret this is that a significant probability mass of the surface of the objective function overlooks the global minimum. It should be noted that the path to the minimum might be difficult to find using local search procedures. Nevertheless, these results shed some light on the nature of the objective function, demonstrating that it is not ``wildly'' non-convex, in the sense that \enquote{crossing valleys} is not a must to reach a good solution, and accords with recent empirical evidence to this effect \citep{goodfellow2014qualitatively}.

For the results here, it would be convenient to re-write the objective function as $L(P(\Wcal))$, where $\Wcal$ is the vector of network parameters, $P(\Wcal)$ is an $m\times k$ matrix, which specifies the prediction for each of the $m$ training points (the prediction can be scalar valued, i.e. $k=1$, or vector-valued when $k>1$), and $L$ is the average loss over the training data. For example, for regression, a standard choice is the squared loss, in which case 
\[
	L(P(\Wcal)) = \frac{1}{m}\sum_{t=1}^{m}(N(\Wcal)(\bx_t)-y_t)^2.
\]
For classification, a standard choice in the context of neural networks is the cross-entropy loss coupled with a softmax activation function, which can be written as $\frac{1}{m}\sum_{t=1}^{m}\ell_t(N(\Wcal)(\bx_t))$, where given a prediction vector $\bp$ and letting $j_t$ be an index of the correct class,
\[
	\ell_t(\bp) = -\log\left(\frac{\exp(p_{j_t})}{\sum_j \exp(p_j)}\right).
\]
Recall that although these losses are convex in the network's predictions, $L(P(\Wcal))$ is still generally non-convex in the network parameters $\Wcal$. Also, we remind that due to the last layer being linear, multiplying its parameters by some scalar $c$ causes the output to change by $c$. Building on this simple observation, we have the following theorem.
\begin{theorem}
	\label{thm:path}
	Suppose $L:\reals^{m\times k}\rightarrow\reals$ is convex. Given a fully-connected  network of any depth, with initialization point $\Wcal^{\p{0}}$, suppose there exists a continuous path $\Wcal^{\p{\lambda}},\lambda\in [0,1]$ in the space of parameter vectors, starting from $\Wcal^{\p{0}}$ and ending in another point $\Wcal^{\p{1}}$ with strictly smaller objective value ($L(P(\Wcal^{\p{1}}))<L(P(\Wcal^{\p{0}}))$), which satisfies the following:
	\begin{enumerate}
		\item \label{assum:mag} For some $\epsilon>0$ and any $\lambda\in [0,1]$, there exists some $c_\lambda\geq 0$ such that $L(c_\lambda\cdot P(\Wcal^{\p{\lambda}}))\geq L(P(\Wcal^{\p{0}}))+\epsilon$.
		\item The initial point $\Wcal^{\p{0}}$ satisfies $L(P(\Wcal^{\p{0}})) > L(\mathbf{0})$.
	\end{enumerate}
	
	Then there exists a continuous path $\tilde{\Wcal}^{\p{\lambda}},\lambda\in [0,1]$ from the initial point $\tilde{\Wcal}^{\p{0}} = \Wcal^{\p{0}}$, to some point $\tilde{\Wcal}^{\p{1}}$ satisfying $L(P(\tilde{\Wcal}^{\p{1}}))=L(P(\Wcal^{\p{1}}))$, along which $L(P(\tilde{\Wcal}^{\p{\lambda}}))$ is strictly monotonically decreasing in $\lambda$.
\end{theorem}

Intuitively, this result stems from the linear dependence of the network's output on the parameters of the last layer. Given the initial non-monotonic path $\Wcal^{\p{\lambda}}$, we rescale the last layer's parameters at each $\Wcal^{\p{\lambda}}$ by some positive factor $c^{(\lambda)}$ depending on $\lambda$ (moving it closer or further from the origin), which changes its output and hence its objective value. We show it is possibly to do this rescaling, so that the rescaled path is continuous and has a monotonically decreasing objective value. In fact, although we focus here on ReLU networks, the theorem itself is quite general and holds even for networks with other activation functions. A formal proof and a more detailed intuition is provided in \subsecref{subsec:proofpath}. 

The first condition in the theorem is satisfied by losses which get sufficiently large (as a function of the network predictions) sufficiently far away from the origin. In particular, it is generally satisfied by both the squared loss and the cross-entropy loss with softmax activations, assuming data points and initialization in general position\footnote{For the squared loss, a sufficient condition is that for any $\lambda$, there is \emph{some} data point on which the prediction of $N(\Wcal^{\p{\lambda}})$ is non-zero. For the cross-entropy loss, a sufficient condition is that for any $\lambda$, there is \emph{some} data point on which $N(\Wcal^{\p{\lambda}})$ outputs an `incorrect' prediction vector $\bp$, in the sense that if $i$ is the correct label, then $p_i \notin \arg\max_j p_j$.}. The second condition requires the random initialization to be such that the initialized network has worse objective value than the all-zeros predictor. However, it can be shown to hold with probability close to $1/2$ (over the network's random initialization), for losses such as those discussed earlier:
\begin{proposition}\label{prop:crosssquare}
	If $L(\cdot)$ corresponds to the squared loss or cross-entropy loss with softmax activation, and the network parameters are initialized as described in Assumption \ref{assum:init_dist}, then 
	\[
		\underset{\Wcal^{\p{0}}}{\mathbb{P}}\pcc{L(P(\Wcal^{\p{0}}))>L(\mathbf{0})} \geq \frac{1}{2}\left(1-2^{-n_{h-1}}\right),
	\]
	where $n_{h-1}$ is the number of neurons in the last layer before the output neurons.
\end{proposition}
This proposition (whose proof appears in appendix \ref{subsec:proofcrosssquare}) is a straightforward corollary of the following result, which can be applied to other losses as well:

\begin{proposition}\label{prop:highinit}
	Suppose the network parameters $\Wcal^{\p{0}}$ are initialized randomly as described in Assumption \ref{assum:init_dist}. Suppose furthermore that  $L(\cdot)$ is such that
	\[
		\mathbb{P}\left[~\text{$L(c\cdot P(\Wcal^{\p{0}}))$ is strictly convex in $c\in [-1,1]$}~~\middle|~~ P(\Wcal^{\p{0}})\neq \mathbf{0}~\right] ~\geq~ r
		\\
	\]
	for some $r>0$ (where the probability is with respect to $\Wcal^{\p{0}}$). Then 
	\[
		\pr{L(P(\Wcal^{\p{0}}))>L(\mathbf{0})}~\geq~ \frac{r}{2}\left(1-2^{-n_{h-1}}\right).
	\]
\end{proposition}
Intuitively, the strict convexity property means that by initializing the neurons from a zero-mean distribution (such as a spherically symmetric one), we are likely to begin at a point with higher objective value than initializing at the mean of the distribution (corresponding to zero weights and zero predictions on all data points). A formal proof appears in Appendix \ref{subsec:prophighinitproof}.

\section{Two-layer ReLU Networks}
\label{sec:Neural_nets_depth_two}
We now turn to consider a more specific network architecture, namely two-layer networks with scalar output. While simpler than deeper architectures, two-layer networks still possess universal approximation capabilities \citep{cybenko1989approximation}, and encapsulate the challenge of optimizing a highly non-convex objective.

From this point onwards, we will consider for simplicity two-layer networks without bias (where $b=0$ for all neurons, not just the output neuron), for the purpose of simplifying our analysis. This is justified, since one could simulate the additional bias term by incrementing the dimension of the data and mapping an instance in the dataset using $\bx\mapsto\p{\bx,1}\in\bbr^{d+1}$, so that the last coordinate of the weight of a neuron will function	as a bias term. Having such a fixed coordinate does not affect the validity of our results for two-layer nets.

We denote our network parameters by $ \p{W,\bv} $ where the rows of the matrix $ W\in\bbr^{n\times d} $ represent the weights of the first layer and $ \bv\in\bbr^n $ represents the weight of the output neuron, and denote a two-layer network of width $ n $ by $ N_n\p{W,\bv}:\bbr^d\to\bbr $. Our objective function with respect to two-layer networks is therefore given by

\[
	L_S\p{W,\bv}~\coleq~ \frac{1}{m}\sum_{t=1}^m\ell\p{N_n\p{W,\bv}\p{\bx_t},y_t}
	~=~ \frac{1}{m}\sum_{t=1}^m\ell\p{\sum_{i=1}^n v_i\cdot\relu{\inner{\bw_i,\bx_t}},y_t},
\]
corresponding to the parameter space $
\set{\p{W,\bv}:W\in\bbr^{n\times d},\bv\in\bbr^n} $.

To say something interesting regarding two-layer nets, we partition our parameter space into regions, inside each the objective function takes a relatively simple form.
Our partition relies on the observation that when considering the subset of our parameter space in which $ \sign \p{\inner{\bw_i,\bx_t}},~\sign \p{v_i} $ are fixed for any neuron $ i $ and any sample instance $ t $, the ReLU activation is then reduced to either the zero function or the identity on $ \inner{\bw_i,\bx_t} $ for all $ i\in\pcc n,t\in\pcc m $, so the objective function takes the form $\frac{1}{m}\sum_{t=1}^m\ell\p{\sum_{i\in I_t} v_i\inner{\bw_i,\bx_t},y_t}$ for some index sets $I_1,\ldots,I_m\subseteq [n]$. This function is not convex or even quasi-convex as a function of $(W,\bv)$. However, it does behave as a basin (as defined in Definition \ref{def:Basin}), and hence contain a single connected set of global minima, with no non-global minima. More formally, we have the following definition and lemma:
\begin{definition}(Basin Partition)
	\label{def:Basin_partition}
	For any $ A\in \set{-1,+1}^{n\times d} $ and $ \bb\in\set{-1,+1}^n $, define $ B_S^{A,\bb} $ as the topological closure of a set of the form
		\[
		\set{(W,v):\forall t\in\pcc m, j\in\pcc n,\sign(\inner{\bw_j, \bx_t})=a_{j,t}, \sign(v_j)=b_j}.
		\]
\end{definition}
We will ignore $ B_S^{A,\bb} $ corresponding to empty sets, since these are irrelevant to our analysis.
\begin{lemma}
	\label{lem:is_a_basin}
	For any $ A\in\set{-1,+1}^{n\times d},~\bb\in\set{-1,+1}^n $ such that $ B_S^{A,\bb} $ is non-empty, $ B_S^{A,\bb} $ is a basin as defined in Definition \ref{def:Basin}.
\end{lemma}
The reader is referred to \appref{app:basin_partition_proof} for the proof of the lemma.

Note that Definition \ref{def:Basin_partition} refers to a partition of the parameter space into a finite number of convex polytopes. Recalling Assumption \ref{assum:init_dist} on the initialization distribution (basically, that it is a Cartesian product of spherically-symmetric distributions), it is easy to verify that we will initialize in an interior of a basin with probability $ 1 $. Therefore, we may assume that we always initialize in some unique basin.

We now focus on understanding when are we likely to initialize at a basin with a low minimal value (which we refer to as the basin value). We stress that this is a purely geometric argument on the structure on the objective function. In particular, even though every local minimum in a basin is also global on the basin, it does not necessarily entail that an optimization algorithm such as stochastic gradient descent will necessarily converge to the basin's global minima (for example, it may drift to a different basin). However, we believe this type of geometric property is indicative of the optimization susceptibility of the objective function, and provides some useful insights on its structure.

We now turn to state a simple but key technical lemma, which will be used to prove the results presented later in this section. Moreover, this lemma also provides some insight into the geometry of the objective function for two-layer networks:
\begin{lemma}
	\label{lem:key_lemma}
	Let $N_n\left(W,\bv\right)$ denote a two-layer network of size $n$, and let $$\p{W,\bv}=\p{\bw_1,\dots,\bw_n,v_1,\dots,v_n}\in \reals^{nd+n}$$ be in the interior of some arbitrary basin. Then for any subset $I=\p{i_1,\dots,i_k}\subseteq\pcc{n}$ we have $$ \bas\p{W,\bv} \le \bas\p{\bw_{i_1},\dots,\bw_{i_k},v_{i_1},\dots,v_{i_k}}. $$
	Where the right hand side is with respect to an architecture of size $ k $.
\end{lemma}

This lemma captures in a way the power overspecification has in the context of two-layer networks: In terms of basin values, any initialization made using a network of width $n\ge k$ (i.e. with $n$ neurons in the first layer) is at least as good as if we had used only a width $k$ network. This is because in our definition of the basin partition, clamping the weights of any $n-k$ neurons to $0$ still keeps us in the same basin, while only increasing the minimal value we can obtain using the $k$ non-clamped neurons. Therefore, if we had only a $k$-width network to begin with, the corresponding basin value can only be larger. We refer the reader to \appref{app:key_lemma_proof} for the proof of the lemma.

\subsection{Bad Local Minima Results: Brittleness to Overspecification}
	The training objective function of neural network is known to be highly non-convex, even for simple networks. A classic and stark illustration of this was provided in \citep{auer1996exponentially} who showed that even for a network comprised of a single neuron (with certain types of non-ReLU activation functions, and with or without bias), the objective function can contain a huge number of local minima (exponentially many in the input dimension). In \appref{app:eps_realizability_single_neurons}, we provide an extension of this result by proving that with a similar construction, and for a neuron with ReLU activation, not only is the number of local minima very large, but the probability of initializing at a basin with good local minimum (using the natural analogue of the basin partition from Definition \ref{def:Basin_partition} for a single neuron) is exponentially small in the dimension. 
	
	The construction provided in \citep{auer1996exponentially} (as well as the one provided in \appref{app:eps_realizability_single_neurons}) relies on training sets $S$ comprised of singleton instances $ \bx_t $, which are non-zero on a single coordinate. The objective function for a single ReLU neuron without bias can be written as $ \sum_{t=1}^m \ell \p{\relu{\inner{\bw,\bx_t}},y_t} $, so if the $ \bx_t $'s are singletons, this can be written as a sum of functions, each depending only on a single coordinate of $ \bw $. The training examples are chosen so that along each coordinate, there are two basins and two distinct local minima, one over the positive values and one over the negative values, but only one of these minima is global. Under the initialization distribution considered, the probabilities of hitting the good basin along each coordinate are independent and strictly less than $1$. Therefore, with overwhelming probability, we will \enquote{miss} the right basin on a constant portion of the coordinates, resulting in a basin value which is suboptimal by at least some constant.
	
	It is natural to study what happens to such a hardness construction under overspecification, which here means replacing a single neuron by a two-layer network of some width $ n>1 $, and training on the same dataset. Surprisingly, it turns out that in this case, the probability of reaching a sub-optimal basin decays exponentially in $ n $ and becomes arbitrarily small already when $n=\Omega\p{\log\p{d}}$. Intuitively, this is because for such constructions, for each coordinate it is enough that \emph{one} of the $ n $ neurons in the first layer will have the corresponding weight initialized in the right basin. This will happen with overwhelming probability if $ n $ is moderately large. More formally, we have the following theorem:
	\begin{theorem}
		\label{thm:brittleness_to_overspecification}
		For any $ n $, let $ \alpha $ denote the minimal objective value achievable with a width $n$ two-layer network, with respect to a convex loss $ \ell $ on a training set $S$ where each $\bx_t$ is a singleton. Then when initializing $ \p{W,\bv}\in \reals^{n\times d}\times \reals^n$ from a distribution satisfying Assumption \ref{assum:init_dist}, we have
		\[
			\pr{\bas\p{W,\bv}\le\alpha} \ge 1-2d\p{\frac{3}{4}}^n.
		\]
	\end{theorem}
	The reader is referred to Appendix \ref{app:brittleness_to_overspecification} for the full proof. 
	
	We note that $ \alpha $ cannot be larger than the optimal value attained using a single neuron architecture. Also, we emphasize that the purpose of \thmref{thm:brittleness_to_overspecification} is not to make a statement about neural networks for singleton datasets (which are not common in practice), but rather to demonstrate the brittleness of hardness constructions such as in \citep{auer1996exponentially} to overspecification, as more neurons are added to the first layer. This motivates us in further studying overspecification in the following subsections.

\subsection{Data With Low Intrinsic Dimension}

	We now turn to provide a result, which demonstrates that for any dataset which is realizable using a two-layer network of a given width $ n $ (i.e. $ L_S\p{N_n\p{W,\bv}}=0 $ for some $ \p{W,\bv} $), the probability of initializing from a basin containing a good minimum increases as we add more neurons to the first layer, corresponding to the idea of overspecification. We note that this result holds without significant additional assumptions, but on the flip side, the number of neurons required to guarantee a constant success probability increases exponentially with the intrinsic dimension of the data ($\text{rank}\p{X}$, where $ X $ is the data matrix whose rows are $ \bx_1,\dots,\bx_m $), so a magnitude of $ \Omega\p{n^{\rank\p{X}}} $ neurons is required. Thus, the result is only meaningful when the intrinsic dimension and $n$ are modest. In the next subsection, 
	we provide results which require a more moderate amount of overspecification, under other assumptions. 
	
	To avoid making the result too complex, we will assume for simplicity that we use the squared loss $\ell(y,y')=(y-y')^2$ and that $\norm{\bx_t}\leq 1$ for any training instance $\bx_t$. However, an analogous result can be shown for any convex loss, with somewhat different dependencies on the parameters, and any other bound on the norms of the instances.
	\begin{theorem}
		\label{thm:low_intrinsic_data_dim}
		Assume each training instance $\bx_t$ satisfies $\norm{\bx_t}\le 1$. Suppose that the training objective $L_S$ refers to the average squared loss, and that $ L_S\p{W^*,\bv^*}=0 $ for some $ \p{W^*,\bv^*}\in \bbr^{n\times d}\times \bbr^n $ satisfying
		\[
			\abs{v^*_i}\cdot\norm{\bw^*_i}\le B ~~ \forall i\in \pcc n,
		\]
		where $ B $ is some constant. For all $ \epsilon>0 $, if
		\begin{align*}
			p_\epsilon~&=~\frac{1}{2\pi\p{\rank\p{X}-1}}\p{\frac{\sqrt{\epsilon}}{nB}\sqrt{1-\frac{\epsilon}{4n^2 B^2}}}^{\rank\p{X}-1}\\
			&=~\Omega\p{\p{\frac{\sqrt \epsilon}{nB}}^{\text{rank}\p{X}}},
		\end{align*}
		and we initialize a two-layer, width $c\ceil{\frac{n}{p_\epsilon}}$ network (for some $c\geq 2$), using a distribution satisfying Assumption \ref{assum:init_dist}, then
		\[
			\pr{\bas\p{W,\bv}\le \epsilon}\ge 1-e^{-\frac{1}{4}cn}.
		\]
	\end{theorem}

	The proof idea is that with a large enough amount of overspecification, with high probability, there will be a subset of the neurons in the first layer for which the signs of their outputs on the data and the signs of their weights in the output neuron will resemble those of $\p{W^*,\bv^*}$. Then, by using \lemref{lem:key_lemma} we are able to argue that the initialization made in the remaining neurons does not degrade the value obtained in the aforementioned subset. We refer the reader to \appref{app:low_intrinsic_data_dim_proof} for the full proof.
	
\subsection{Clustered or Full-rank Data}
	In this subsection, we will first show that when training on instances residing in high dimension $d$ (specifically, when the dimension satisfies $m\le d$, where $m$ is the number of training examples), we initialize at a good basin with high probability. Building on this result, we show that even when $m>d$, we still initialize at a good basin with high probability, as long as the data is clustered into $k\leq d$ sufficiently small clusters.	
	
	Specifically, we begin by assuming that our data matrix $X$ satisfies $\rank \p{X}=m$. We note that this immediately implies $m\le d$. This refers to data of very high intrinsic dimension, which is in a sense the opposite regime to the one considered in the previous subsection (where the data was assumed to have low intrinsic dimension). Even though this regime might be strongly prone to overfitting, this allows us to investigate the surface area of	the objective function effectively, while also serving as a base for the clustered data scenario that we will be studying in \thmref{thm:Clustered_data}.
	
	We now state our formal result for such datasets, which implies that under the rank assumption, a two-layer network of size $\mathcal{O}\p{\log\p m}$ is sufficient to initialize in a basin with a global minimum with overwhelming probability.
	\begin{theorem}
		\label{thm:Rank_m_data}
		Assume $\rank\p X=m$, and let the target outputs $y_1,\ldots,y_m$ be arbitrary. For any $ n $, let $\alpha$ be the minimal objective value achievable with a width $n$ two-layer network. Then if $ \p{W,\bv}\in\bbr^{n\times d}\times\bbr^n $ is initialized according to Assumption \ref{assum:init_dist},
		\[
			\pr{\bas\p{W,\bv}\leq \alpha}
			~\ge~ 1-m\p{\frac{3}{4}}^n.
		\]
	\end{theorem}
	
	We refer the reader to \appref{app:Rank_m_data} for the full proof of the theorem.
	
	As mentioned earlier, training on $m\le d$ examples, without imposing any regularization, is prone to overfitting. Thus, to say something meaningful in the $m>d$ regime, we will consider an extension of the previous result, where instead of having fewer data points than dimensions $d$, we assume that the training instances are composed of $k\le d$ relatively small clusters in general position. Intuitively, if the clusters are sufficiently small, the surface of the objective function will resemble that of having $k\le d$ data points, and will have a similar favorable structure.
	
	We also point out that in a similar manner to as we did in \thmref{thm:low_intrinsic_data_dim}, the theorem statement assumes that the objective function refers to the average squared loss over the data. However, the proof does not rely on special properties of this loss, and it is possible to generalize it to other convex losses (perhaps with a somewhat different resulting bound).

	\begin{theorem}
		\label{thm:Clustered_data}
Consider the squared loss, and suppose our data is clustered into $k\le d$	clusters. Specifically, we assume there are cluster centers $\bc_1,\dots,\bc_k\in \reals^d$ for which the training data $ S=\set{\bx_t,y_t}_{t=1}^m $ satisfies the following:
		\begin{itemize}
			\item $\exists\delta_1,\dots,\delta_k>0$
			s.t. for all $\bx_t$, there is a unique $j\in [k]$ such that $\norm{\bc_j-\bx_t}\le\delta_j$.
			\item $\forall j\in\pcc k\;\frac{\delta_j}{\norm{\bc_j}}\le2\sin\p{\frac{\sqrt{2\pi}}{16d\sqrt{d}}}$ and $\forall j\in\pcc k\text{ }\norm{\bc_j}\ge c$ for some
			$c>0$.
			\item $\forall t\in\pcc m\;\norm{\bx_t}\le B$ for some $B\in\bbr$.
			\item For some fixed $\gamma$, it holds that $ \abs{y_t-y_{t^\prime}}\le \gamma\norm{\bx_t-\bx_{t^\prime}}_2 $ for any $t,t^\prime\in\pcc m$ such that $ \bx_t,\bx_{t^\prime} $ are in the same cluster.
		\end{itemize}
		
		Let $\delta=\max_j \delta_j$. Denote as $C$ the matrix which rows are $\bc_1,\dots,\bc_k$,
		and let $ \sigma_{\max}\p{C^\top},\sigma_{\min}\p{C^\top} $ denote the largest and smallest singular values of $ C^\top $ respectively. Let $\mathsf{c}\p{\bx_t}:\bbr^d\to\bbr^d$ denote the mapping of $ \bx_t $ to its nearest cluster center $ \bc_j $ (assumed to be unique), and finally, let $\hat{\by}=\p{\hat{y}_{1},\dots,\hat{y}_k}\in\bbr^k$ denote the target values of arbitrary instances from each of the $ k $ clusters. Then if $ \p{W,\bv}\in\bbr^{n\times d}\times\bbr^n  $ is initialized from a distribution satisfying Assumption \ref{assum:init_dist},
		\[
			\pr{\bas\p{W,\bv}\le\Ocal\p{\delta^2}}\ge1-d\p{\frac{7}{8}}^n
		\]
		Where the big $ \Ocal $ notation hides quadratic dependencies on $ B,c^{-1},n,\sigma_{\min}^{-2}\p{C^\top},\sigma_{\max}\p{C^\top},\gamma,\norm{\hat{\by}}_2 $ (see the proof provided in \appref{app:Clustered_data_proof} for an explicit expression).
	\end{theorem}
	Note that $\delta$ measures how tight the clusters are, whereas $c,\sigma_{\max}\p{C^\top}$ and $\sigma_{\min}\p{C^\top}$ can be thought of as constants assuming the cluster centers are in general position. So, the theorem implies that for sufficiently tight clusters, with overwhelming probability, we will initialize from a basin containing a low-valued minimum, as long as the network size is $\Omega\p{\log\p d}$.

\subsubsection*{Acknowledgements}
This research is supported in part by an FP7 Marie Curie CIG grant, Israel Science Foundation grant 425/13, and the Intel ICRI-CI Institute. We thank Lukasz Kaiser for pointing out a bug (as well as the fix) in an earlier version of the paper.

\bibliography{citations}

\begin{thebibliography}{17}
\providecommand{\natexlab}[1]{#1}
\providecommand{\url}[1]{\texttt{#1}}
\expandafter\ifx\csname urlstyle\endcsname\relax
  \providecommand{\doi}[1]{doi: #1}\else
  \providecommand{\doi}{doi: \begingroup \urlstyle{rm}\Url}\fi

\bibitem[Andoni et~al.(2014)Andoni, Panigrahy, Valiant, and
  Zhang]{andoni2014learning}
Alexandr Andoni, Rina Panigrahy, Gregory Valiant, and Li~Zhang.
\newblock Learning polynomials with neural networks.
\newblock In \emph{ICML}, 2014.

\bibitem[Arora et~al.(2014)Arora, Bhaskara, Ge, and Ma]{arora2014provable}
Sanjeev Arora, Aditya Bhaskara, Rong Ge, and Tengyu Ma.
\newblock Provable bounds for learning some deep representations.
\newblock In \emph{Proceedings of The 31st International Conference on Machine
  Learning}, pages 584--592, 2014.

\bibitem[Auer et~al.(1996)Auer, Herbster, and Warmuth]{auer1996exponentially}
Peter Auer, Mark Herbster, and Manfred~K Warmuth.
\newblock Exponentially many local minima for single neurons.
\newblock In \emph{NIPS}, 1996.

\bibitem[Bach(2014)]{bach2014breaking}
Francis Bach.
\newblock Breaking the curse of dimensionality with convex neural networks.
\newblock \emph{arXiv preprint arXiv:1412.8690}, 2014.

\bibitem[Bengio et~al.(2005)Bengio, Roux, Vincent, Delalleau, and
  Marcotte]{bengio2005convex}
Yoshua Bengio, Nicolas~L Roux, Pascal Vincent, Olivier Delalleau, and Patrice
  Marcotte.
\newblock Convex neural networks.
\newblock In \emph{Advances in neural information processing systems}, pages
  123--130, 2005.

\bibitem[Blum and Rivest(1992)]{blum1992training}
Avrim~L Blum and Ronald~L Rivest.
\newblock Training a 3-node neural network is np-complete.
\newblock \emph{Neural Networks}, 5\penalty0 (1):\penalty0 117--127, 1992.

\bibitem[Choromanska et~al.(2014)Choromanska, Henaff, Mathieu, Arous, and
  LeCun]{choromanska2014loss}
Anna Choromanska, Mikael Henaff, Michael Mathieu, G{\'e}rard~Ben Arous, and
  Yann LeCun.
\newblock The loss surface of multilayer networks.
\newblock \emph{arXiv preprint arXiv:1412.0233}, 2014.

\bibitem[Cybenko(1989)]{cybenko1989approximation}
George Cybenko.
\newblock Approximation by superpositions of a sigmoidal function.
\newblock \emph{Mathematics of control, signals and systems}, 2\penalty0
  (4):\penalty0 303--314, 1989.

\bibitem[Dauphin et~al.(2014)Dauphin, Pascanu, Gulcehre, Cho, Ganguli, and
  Bengio]{dauphin2014identifying}
Y.~Dauphin, R.~Pascanu, C.~Gulcehre, K.~Cho, S.~Ganguli, and Y.~Bengio.
\newblock Identifying and attacking the saddle point problem in
  high-dimensional non-convex optimization.
\newblock In \emph{NIPS}, 2014.

\bibitem[Goodfellow and Vinyals(2014)]{goodfellow2014qualitatively}
Ian~J Goodfellow and Oriol Vinyals.
\newblock Qualitatively characterizing neural network optimization problems.
\newblock \emph{arXiv preprint arXiv:1412.6544}, 2014.

\bibitem[Haeffele and Vidal(2015)]{haeffele2015global}
Benjamin~D Haeffele and Ren\'{e} Vidal.
\newblock Global optimality in tensor factorization, deep learning, and beyond.
\newblock \emph{arXiv preprint arXiv:1506.07540}, 2015.

\bibitem[Janzamin et~al.(2015)Janzamin, Sedghi, and
  Anandkumar]{janzamin2015beating}
Majid Janzamin, Hanie Sedghi, and Anima Anandkumar.
\newblock Beating the perils of non-convexity: Guaranteed training of neural
  networks using tensor methods.
\newblock \emph{CoRR abs/1506.08473}, 2015.

\bibitem[Kumagai(1980)]{kumagai1980implicit}
Sadatoshi Kumagai.
\newblock An implicit function theorem: Comment.
\newblock \emph{Journal of Optimization Theory and Applications}, 31\penalty0
  (2):\penalty0 285--288, 1980.

\bibitem[Leopardi(2007)]{leopardi2007distributing}
Paul Leopardi.
\newblock \emph{Distributing points on the sphere: partitions, separation,
  quadrature and energy}.
\newblock PhD thesis, University of New South Wales, 2007.

\bibitem[Li(2011)]{li2011concise}
Shengqiao Li.
\newblock Concise formulas for the area and volume of a hyperspherical cap.
\newblock \emph{Asian Journal of Mathematics and Statistics}, 4\penalty0
  (1):\penalty0 66--70, 2011.

\bibitem[Livni et~al.(2014)Livni, Shalev-Shwartz, and
  Shamir]{livni2014computational}
Roi Livni, Shai Shalev-Shwartz, and Ohad Shamir.
\newblock On the computational efficiency of training neural networks.
\newblock In \emph{NIPS}, pages 855--863, 2014.

\bibitem[Zhang et~al.(2015)Zhang, Lee, Wainwright, and
  Jordan]{zhang2015learning}
Yuchen Zhang, Jason~D Lee, Martin~J Wainwright, and Michael~I Jordan.
\newblock Learning halfspaces and neural networks with random initialization.
\newblock \emph{arXiv preprint arXiv:1511.07948}, 2015.

\end{thebibliography}
\bibliographystyle{plainnat}

\newpage
\appendix

\section{Proofs of Basin Partition Properties}
\subsection{Proof of \lemref{lem:is_a_basin}}
	\label{app:basin_partition_proof}
	
	We will need the following three auxiliary lemmas.
	\begin{lemma}
		\label{lem:convexity_on_basins}
		Let $ B $ be some basin as defined in Definition \ref{def:Basin_partition}, and define $ \bz_i = v_i\bw_i $. Then $$ L_S\p{Z} = L_S\p{W,\bv} $$ is convex in $ Z=\p{\bz_1,\dots,\bz_n} $ on $ B $.
	\end{lemma}
	
	\begin{proof}
		Restricting ourselves to $ B $, since $ \sign\p{\inner{\bw_j,\bx_t}},~\sign\p{v_j} $ are fixed, we can rewrite our objective as
		\[
		\frac{1}{m}\sum_{t=1}^{m}\ell\p{\sum_{i=1}^n\sigma_{i,t}\inner{v_i \bw_i,\bx_t},y_t} = \frac{1}{m}\sum_{t=1}^{m}\ell\p{\sum_{i=1}^n\sigma_{i,t}\inner{\bz_i,\bx_t},y_t},
		\]
		where $ \sigma_{i,t}\in\set{-1,0,+1} $ are fixed. This is a linear function composed with a convex loss $ \ell $, therefore the objective is convex in $ Z $.
	\end{proof}
	
	\begin{lemma}
		\label{lem:second_layer_rescaling}
		Let $ \p{W,\bv}\in B_S^{A,\bb} $. There exists a continuous path $ \p{\tilde{W}^{\p{\lambda}},\tilde{\bv}^{\p{\lambda}}},\; \lambda\in\pcc{0,1} $ from the initial point $ \p{\tilde{W}^{\p{0}},\tilde{\bv}^{\p{0}}} = \p{W,\bv} $, to a point $ \p{\tilde{W}^{\p{1}},\tilde{\bv}^{\p{1}}} $ satisfying $ \tilde{\bv}^{\p{1}}=\bb $, along which $ N_n\p{\tilde{W}^{\p{\lambda}},\tilde{\bv}^{\p{\lambda}}} $ is constant and $ \p{\tilde{W}^{\p{\lambda}},\tilde{\bv}^{\p{\lambda}}}\in B_S^{A,\bb}\;\;\forall\lambda\in\pcc{0,1}.$
	\end{lemma}
	
	\begin{proof}
		If $ v_i=0 $ for some $ i\in\pcc{n} $, then the $ i^\text{th} $ neuron is canceled and we can linearly rescale $ \bw_i $ to $ \mathbf{0} $, and then rescale $ v_i $ to $ b_i $, so we may assume without loss of generality that $ v_i\neq0 $ for all $ i\in \pcc{n} $. We have for all $ \balpha = \p{\alpha_1,\dots,\alpha_n} \succ \mathbf{0}$,
		\begin{align*}
			N_n\p{W,\bv}\p{\bx} 
			&= \sum_{i=1}^n v_i \relu{\inner{\bw_i,\bx}} \\
			&= \sum_{i=1}^n \frac{v_i}{\alpha_i} \alpha_i \relu{\inner{\bw_i,\bx}} \\
			&= \sum_{i=1}^n \frac{v_i}{\alpha_i} \relu{\inner{\alpha_i\bw_i,\bx}}.
		\end{align*}
		Where we used the positive homogeneity of $ \relu{\cdot} $ in the last equality. So by linearly scaling $\balpha^{\p{0}}= \p{1,\dots,1} $ to $ \balpha^{\p{1}}=\p{\abs{v_1},\dots,\abs{v_n}} $, i.e. $ \balpha^{\p{\lambda}}=\p{1-\lambda+\lambda\abs{v_1},\dots,1-\lambda+\lambda\abs{v_n}},\; \lambda\in\pcc{0,1} $, we obtain the desired path
		\[
			\tilde{W}^{\p{\lambda}} = \p{\alpha_{1}^{\p{\lambda}}\bw_1,\dots,\alpha_{n}^{\p{\lambda}}\bw_n,},
		\]
		\[
			\tilde{\bv}^{\p{\lambda}} = \p{\frac{v_1}{\alpha_{1}^{\p{\lambda}}},\dots,\frac{v_n}{\alpha_{n}^{\p{\lambda}}}},
		\]
		while noting that $ \sign\p{v_i} = \sign\p{\frac{v_i}{\alpha_i}} $ and $ \sign\p{\inner{\bw_i,\bx}} = \sign\p{\inner{\alpha_i\bw_i,\bx}} $ for all $ \balpha \succ 0 $, therefore we remain inside $ B_S^{A,\bb} $.
	\end{proof}
	
	\begin{lemma}
		\label{lem:basin_path_convergence}
		For $ \p{W,\bv},\p{\tilde{W},\tilde{\bv}} \in B_S^{A,\bb}$, define
		$$ v_i^{\p{\lambda}} = \lambda\tilde{v}_i+\p{1-\lambda}v_i, $$
		$$ \bw_i^{\p{\lambda}}=
		\begin{cases*}
		 \lambda\tilde{\bw}_i+\p{1-\lambda}\bw_i & $v_i=\tilde{v}_i=0$ \\
		 \lambda\frac{\tilde{v}_i\tilde{\bw}_i}{v_i^{\p{\lambda}}} + \p{1-\lambda}\frac{v_i\bw_i}{v_i^{\p{\lambda}}} & \text{otherwise}
		\end{cases*}.
		$$
		Then 
		
		\begin{enumerate}
			\item
			\label{lem:part1_product}
			$ v_i^{\p{\lambda}}\bw_i^{\p{\lambda}} = \lambda\tilde{v}_i \tilde{\bw}_i + \p{1-\lambda}v_i\bw_i ~~  \forall i\in\pcc n,\lambda\in\p{0,1}. $
			\item
			\label{lem:part2_convergence}
			$ \p{\bw_i^{\p{\lambda}},v_i^{\p{\lambda}}} \underset{\lambda\rightarrow 0_+}{\longrightarrow}  \p{\bw_i,v_i} ~~ \forall i\in\pcc n. $
			\item
			\label{lem:part3_in_basin}
			$ \p{\bw_1^{\p{\lambda}},\dots,\bw_n^{\p{\lambda}},v_1^{\p{\lambda}},\dots,v_n^{\p{\lambda}}} \in B_S^{A,\bb} ~~ \forall \lambda\in\p{0,1}. $
		\end{enumerate}
	\end{lemma}
	
	\begin{proof}~\\
		\begin{enumerate}
			\item
			Can be shown using a straightforward computation.
			\item
			Compute
			\[
			\lim\limits_{\lambda\to 0_+} v_i^{\p{\lambda}} = \lim\limits_{\lambda\to 0_+} \lambda\tilde{v}_i+\p{1-\lambda}v_i = v_i.
			\]
			Suppose $ v_i=\tilde{v}_i=0 $, then
			\[
			\lim\limits_{\lambda\to 0_+} \bw_i^{\p{\lambda}} = \lim\limits_{\lambda\to 0_+} \lambda\tilde{\bw}_i+\p{1-\lambda}\bw_i = \bw_i.
			\]
			Otherwise, $ v_i^{\p{\lambda}}\neq 0 ~~\forall \lambda\in\p{0,1}$ since $ \sign\p{v_i} = \sign\p{\tilde{v}_i} $, and we have
			\begin{align*}
				\lim\limits_{\lambda\to 0_+}\bw_i^{\p{\lambda}} &=
				\lim\limits_{\lambda\to 0_+} \p{\lambda\frac{\tilde{v}_i\tilde{\bw}_i}{v_i^{\p{\lambda}}} + \p{1-\lambda}\frac{v_i\bw_i}{v_i^{\p{\lambda}}}} \\
				&= \lim\limits_{\lambda\to 0_+} \frac{\lambda\tilde{v}_i\tilde{\bw}_i}{\lambda\tilde{v}_i+\p{1-\lambda}v_i}
				+ \lim\limits_{t\to 0_+} \frac{\p{1-\lambda}v_i\bw_i}{\lambda\tilde{v}_i+\p{1-\lambda}v_i} \\
				&= 0 + \frac{v_i\bw_i}{v_i} \\
				&= \bw_i.
			\end{align*}
			\item
			\label{linear_comb_in_basin}
			Since $ \sign\p{\tilde{v}_i}=\sign\p{v_i} $, we have
			\begin{align*}
				\sign\p{v_i^{\p{\lambda}}} &= \sign\p{\lambda\tilde{v}_i+\p{1-\lambda}v_i} \\
				&= \lambda\sign\p{\tilde{v}_i}+\p{1-\lambda}\sign\p{v_i}.
			\end{align*}
			Suppose $ v_i=\tilde{v}_i=0 $, then since $ \sign\p{\inner{\tilde{\bw}_i,\bx_t}} = \sign\p{\inner{\bw_i,\bx_t}} $, we have $ \forall t\in\pcc m, i\in\pcc n, \lambda\in\p{0,1} $
			\begin{align*}
				\sign\p{\inner{\bw_i^{^{\p{\lambda}}},\bx_t}} &= \sign\p{\inner{\lambda\tilde{\bw}_i+\p{1-\lambda}\bw_i,\bx_t}} \\
				&= \sign\p{\lambda\inner{\tilde{\bw}_i,\bx_t}+\p{1-\lambda}\inner{\bw_i,\bx_t}} \\
				&= \lambda\sign\p{\inner{\tilde{\bw}_i,\bx_t}}+\p{1-\lambda}\sign\p{\inner{\bw_i,\bx_t}}.
			\end{align*}
			Otherwise,
			\begin{align*}
				\sign\p{\inner{\bw_i^{\p{\lambda}},\bx_t}} &= \sign\p{\inner{\lambda\frac{\tilde{v}_i\tilde{\bw}_i}{v_i^{\p{\lambda}}} + \p{1-\lambda}\frac{v_i\bw_i}{v_i^{\p{\lambda}}},\bx_t}} \\
				&= \sign\p{\frac{\tilde{v}_i\lambda}{v_i^{\p{\lambda}}}\inner{\tilde{\bw}_i,\bx_t} + \frac{v_i\cdot\p{1-\lambda}}{v_i^{\p{\lambda}}}\inner{\bw_i,\bx_t}} \\
				&= \sign\p{\inner{\bw_i,\bx_t}}.
			\end{align*}
		\end{enumerate}
	\end{proof}
	
	We are now ready to prove Lemma \ref{lem:is_a_basin}.
	\begin{proof}[Proof (of Lemma \ref{lem:is_a_basin})]~\\
		Clearly, $ B_S^{A,\bb} $ is a closed set, and is convex as an intersection of halfspaces.
		\begin{itemize}
			\item
			$ B_S^{A,\bb} $ has connected sublevel sets:\\
			Let $ \p{W,\bv},\p{W^\prime,\bv^\prime}\in B_{\le\alpha} $. Using \lemref{lem:second_layer_rescaling} we may assume without loss of generality that $ \bv=\bv^\prime \in \set{-1,+1}^n $. By linearly interpolating $ W,W^\prime $, i.e. by taking
			\[
				W^{\p{\lambda}}=\p{1-\lambda}W+\lambda W^\prime,\; \lambda\in\pcc{0,1},
			\]
			we get a continuous path connecting $ W,W^\prime $. This path remains in the same basin as a result of \lemref{lem:basin_path_convergence}.\ref{linear_comb_in_basin}. Moreover, by \lemref{lem:convexity_on_basins}, the objective is convex in $ W $, so we get for all $ \lambda\in\pcc{0,1} $
			\begin{align*}
				E_S\p{W^{\p{\lambda}},\bv} &\le \p{1-\lambda}E_S\p{W,\bv}+\lambda E_S\p{W^\prime,\bv}\\
				&\le \p{1-\lambda}\alpha+\lambda\alpha\\
				&= \alpha.
			\end{align*}
			\item
			Any local minimum in $ B_S^{A,\bb} $ is global:\\
			Suppose $ \p{W,\bv} = \p{\bw_1,\dots,\bw_n,v_1\dots,v_n}$ is a local minimum in $ B_S^{A,\bb} $, let $$ \p{\tilde{W},\tilde{\bv}} = \p{\tilde{\bw}_1,\dots,\tilde{\bw}_n,\tilde{v}_1\dots,\tilde{v}_n} \in B_S^{A,\bb} $$ be arbitrary, and denote $$ \p{W^{\p{\lambda}},\bv^{\p{\lambda}}}\coleq \p{\bw_1^{\p{\lambda}},\dots,\bw_n^{\p{\lambda}},v_1^{\p{\lambda}},\dots,v_n^{\p{\lambda}}}. $$
			
			Then for small enough $ \lambda $
			\begin{align*}
				L_S\p{W,\bv} &\le L_S\p{W^{\p{\lambda}},\bv^{\p{\lambda}}} \\
				&= L_S\p{W^{\p{\lambda}}\cdot\bv^{\p{\lambda}}} \\
				&= L_S\p{\lambda\p{\tilde{v}_1\tilde{\bw}_1,\dots,\tilde{v}_n\tilde{\bw}_n} + \p{1-\lambda}\p{v_1\bw_1,\dots,v_n\bw_n}} \\
				&\le \lambda L_S\p{\tilde{v}_1\tilde{\bw}_1,\dots,\tilde{v}_n\tilde{\bw}_n}
				+ \p{1-\lambda}L_S\p{v_1\bw_1,\dots,v_n\bw_n} \\
				&= \lambda L_S\p{\tilde{W},\tilde{\bv}}
				+ \p{1-\lambda}L_S\p{W,\bv},
			\end{align*}
			\[
				\implies L_S\p{W,\bv} \le L_S\p{\tilde{W},\tilde{\bv}}.
			\]
			Where the first transition comes from $ \p{W,\bv} $ being a local minimum and Lemma \ref{lem:basin_path_convergence}.\ref{lem:part2_convergence},\ref{lem:basin_path_convergence}.\ref{lem:part3_in_basin}, the second and third from Lemma \ref{lem:basin_path_convergence}.\ref{lem:part1_product}, and the fourth from Lemma \ref{lem:convexity_on_basins}.
		\end{itemize}
	\end{proof}
	
\subsection{Proof of \lemref{lem:key_lemma}}
	\label{app:key_lemma_proof}
	Let $ \p{W^*,\bv^*} = \p{\bw^*_1,\dots,\bw^*_k,v^*_1,\dots,v^*_k} \in \bbr^{kd+k} $ satisfy $$ \bas\p{\bw_{i_1},\dots,\bw_{i_k},v_{i_1},\dots,v_{i_k}} = L_S\p{N_k\p{W^*,\bv^*}}, $$ and let $$ \p{W^\prime,\bv^\prime} = \p{\bw^\prime_1,\dots,\bw^\prime_n,v^\prime_1,\dots,v^\prime_n} \in \bbr^{nd+n}, $$ where 
	$$ \p{\bw^\prime_i,v^\prime_i} = 
	\begin{cases} 
		0 & i\notin I \\
		\p{\bw^*_j,v^*_j} & i=i_j \\
	\end{cases}. $$ Then
	
	\begin{align*}
		\bas\p{W,\bv} &\le L_S\p{N_n\p{W^\prime,\bv^\prime}} \\
		&= L_S\p{N_n\p{\bw^\prime_1,\dots,\bw^\prime_n,v^\prime_1,\dots,v^\prime_n}} \\
		&= L_S\p{N_k\p{\bw^*_1,\dots,\bw^*_k,v^*_1,\dots,v^*_k}} \\
		&= L_S\p{N_k\p{W^*,\bv^*}} \\
		&= \bas\p{\bw_{i_1},\dots,\bw_{i_k},v_{i_1},\dots,v_{i_k}}.
	\end{align*}
	Where the first inequality comes from $ \p{W,\bv},\p{W^\prime,\bv^\prime} $ belonging to the same basin, and the second equality comes from both weights computing the same network output for any input $ \bx\in\bbr^d $.
	
\section{Proofs of Main Theorems}
\subsection{Proof of \thmref{thm:path}}\label{subsec:proofpath}

	Before delving into the proof of the theorem, we provide some intuition in the special case of the squared loss, where $L(P(\Wcal)) = \frac{1}{m}\sum_{t=1}^{m}(N(\Wcal)(\bx_t)-y_t)^2$. Fix some $\lambda\in [0,1]$, and consider the objective function along the ray in the parameter space, corresponding to multiplying the last layer weights in $\Wcal^{\p{\lambda}}$ by some scalar $c\geq0 $. Since the output layer is linear, the objective function (as we vary $c$) will have the form
	\[
	\frac{1}{m}\sum_{t=1}^{m}(c\cdot N(\Wcal^{\p{\lambda}})(\bx_t)-y_t)^2.
	\]
	Thus, the objective function, as a parameter of $c$ (where $\Wcal^{\p{\lambda}}$ is fixed) is just a quadratic function. Moreover, by the intermediate value theorem, as long as $N(\Wcal^{\p{\lambda}})(\bx_t)$ is not $0$ for all $t$, then by picking different values of $c$, we can find points along the ray taking any value between $\frac{1}{m}\sum_{i=1}^{t} y_t^2$ (when $c=0$) and $\infty$ (as $c\rightarrow \infty$). Therefore, as long as we start from a point whose objective value is larger than $\frac{1}{m}\sum_{i=1}^{t} y_t^2$, we can re-scale each $\Wcal^{\p{\lambda}}$ by some factor $c_{\gamma}$, so that the new path is continuous, as well as monotonically decreasing in value, remaining above $\frac{1}{m}\sum_{i=1}^{t} y_t^2$. When we reach the ray belonging to the endpoint $\Wcal^{\p{1}}$ of the original path, we simply re-scale back towards $\Wcal^{\p{1}}$, with the objective function continuing to decrease due to the convex quadratic form of the objective function along the ray.

	
	We now turn to the formal proof in the general setting of \thmref{thm:path}. For technical reasons, we will extend the interval $\lambda\in [0,1]$ to a strictly larger interval, and define certain quantities with respect to that larger interval. Specifically, for any $\lambda\in [-1,2]$, define
	\[
		v^{\p{\lambda}} = 
		\begin{cases}
			L(P(\Wcal^{\p{0}}))-\frac{\lambda}{2}\epsilon & \lambda\in [-1,0]\\
			\left(1-\frac{\lambda}{3}\right)\cdot L(P(\Wcal^{\p{0}}))+\frac{\lambda}{3}\cdot \max\{L(\mathbf{0}),L(P(\Wcal^{\p{1}}))\} & \lambda\in [0,2].
		\end{cases}
	\]
	and note that it strictly monotonically decreases with $\lambda$, and satisfies the chain of inequalities
	\[
	L(P(\Wcal^{\p{0}}))+\epsilon > v^{\p{-1}} > v^{\p{0}} = L(P(\Wcal^{\p{0}})) > v^{\p{2}} > \max\{L(\mathbf{0}),L(P(\Wcal^{\p{1}}))\}.
	\]
	
	By assumption, for any $\lambda\in [0,1]$, there exists some $c^{\p{\lambda}}$ such that 
	$L(c^{\p{\lambda}} \cdot P(\Wcal^{\p{\lambda}}))\geq L(P(\Wcal^{\p{0}}))+\epsilon$. Since $L(P(\Wcal^{\p{0}}))+\epsilon> v^{\p{\lambda}}$ for any $\lambda\in [-1,2]$, it follows that for any such $\lambda$, 
	\begin{equation}\label{eq:upbound}
	L(c^{\text{clip}(\lambda)}\cdot P(\Wcal^{\text{clip}(\lambda)}))> v^{\p{\lambda}},
	\end{equation}
	where $\text{clip}(\lambda)=\min\{1,\max\{0,\lambda\}\}$ denote clipping of $\lambda$ to the interval $[0,1]$.
	On the other hand, for any $\lambda\in [-1,2]$,
	\begin{equation}\label{eq:lowbound}
	L(0\cdot P(\Wcal^{\text{clip}(\lambda)})) = L(\mathbf{0}) < v^{\p{\lambda}}.
	\end{equation}
	Since $L$ is convex and real-valued, it is continuous, hence $L(c\cdot P(\Wcal^{\text{clip}(\lambda)}))$ is convex and continuous in $c$. 	Combining this with \eqref{eq:upbound} and \eqref{eq:lowbound}, it follows from the intermediate value theorem that 
	\begin{equation}\label{eq:tct}
	\forall \lambda\in [-1,2],~ ~\exists~\tilde{c}^{\p{\lambda}}\in (0,c^{\text{clip}(\lambda)})~~\text{such that}~~ L(\tilde{c}^{\p{\lambda}}\cdot P(\Wcal^{\text{clip}(\lambda)})=v^{\p{\lambda}}.
	\end{equation}
	Moreover, this $\tilde{c}^{\p{\lambda}}$ is unique: To see why, consider the convex function $f(c)=L(c \cdot P(\Wcal^{\text{clip}(\lambda)})$, and assume there are two distinct values $c',c$ such that $c^{\text{clip}(\lambda)}> c'> c> 0$, and $f(c)=f(c')=v^{\p{\lambda}}$. Then by \eqref{eq:upbound} and \eqref{eq:lowbound}, we would have the chain of inequalities
	\[
	f(c^{\text{clip}(\lambda)})> f(c')=f(c)>f(0)
	\]
	which cannot be satisfied by a convex function $f$.
	
	We now make the following series of observations on $\tilde{c}^{\p{\lambda}}$, which establish their continuity in $\lambda$ and that $\tilde{c}^{\p{0}}=1$:
	\begin{enumerate}
		\item \label{obs:2}\emph{For any $\lambda\in (-1,2)$, there is some open neighborhood of $\tilde{c}^{\p{\lambda}}$ in which the univariate function $c\mapsto L(c\cdot P(\Wcal^{\text{clip}(\lambda)}))$ is one-to-one}: As discussed above, $c\mapsto L(c\cdot P(\Wcal^{\text{clip}(\lambda)}))$ is convex, and does not attain a minimal value at $\tilde{c}^{\p{\lambda}}$ (since $L(\tilde{c}^{\p{\lambda}}\cdot P(\Wcal^{\text{clip}(\lambda)}))= v^{\p{\lambda}}>L(0\cdot P(\Wcal^{\text{clip}(\lambda)}))$). Therefore, $L(c\cdot P(\Wcal^{\text{clip}(\lambda)}))$ must be strictly increasing or decreasing in some open neighborhood of $\tilde{c}^{\p{\lambda}}$, and therefore it is locally one-to-one.
		\item \label{obs:3}\emph{$\tilde{c}^{\p{\lambda}}$ is continuous in $\lambda\in (-1,2)$}: Consider the function $f(c,\lambda)=L(c\cdot P(\Wcal^{\text{clip}(\lambda)}))-v^{\p{\lambda}}$, where $(c,\lambda)\in (0,\infty)\times (-1,2)$. By definition of $\tilde{c}^{\p{\lambda}}$, we have $f(\tilde{c}^{\p{\lambda}},\lambda)=0$ for all $\lambda$. Moreover, $f$ is continuous (since $P(\Wcal^{\text{clip}(\lambda)}),v^{\p{\lambda}}$ are continuous in $\lambda$, and $L$ is convex and hence continuous), and by observation \ref{obs:2} above, $f(\cdot,\lambda)$ is locally one-to-one for any $\lambda\in (-1,2)$. Applying a version of the implicit function theorem from multivariate calculus for possibly non-differentiable functions (see \citep{kumagai1980implicit}), it follows that there exists some unique and continuous mapping of $\lambda$ to $c$ in an open neighborhood of $(\tilde{c}^{\p{\lambda}},\lambda)$, for which $f(c,\lambda)=0$. 
		But as discussed earlier, for a given $\lambda$, $c=\tilde{c}^{\p{\lambda}}$ is the unique value for which $f(c,\lambda)=L(c\cdot P(\Wcal^{\text{clip}(\lambda)}))-v^{\p{\lambda}}=0$, so this continuous mapping must map $\lambda$ to $\tilde{c}^{\p{\lambda}}$ locally at every $\lambda$. Since this holds for any $\lambda$, the mapping of $\lambda$ to $\tilde{c}^{\p{\lambda}}$ is continuous on $\lambda\in (-1,2)$. 
		\item \emph{$\tilde{c}_{0}=1$}: By definition of $\tilde{c}^{\p{\lambda}}$ and $v^{\p{\lambda}}$ at $\lambda=0$, $ L(\tilde{c}^{\p{0}}\cdot P(\Wcal^{\p{0}}))=v^{\p{0}}=L(P(\Wcal^{\p{0}}))$, which is clearly satisfied for $\tilde{c}^{\p{0}}=1$, and as discussed earlier, is not satisfied for any other value.
	\end{enumerate}
	Based on the above observations, we have that $\tilde{c}^{\p{\lambda}}$, as a function of $\lambda\in [0,1]$, is continuous, begins at $\tilde{c}_0=1$, and satisfies $L(\tilde{c}^{\p{\lambda}}\cdot P(\Wcal^{\p{\lambda}}))=v^{\p{\lambda}}$. Moreover, $v^{\p{\lambda}}$ is strictly decreasing in $\lambda$. Therefore, letting
	\begin{equation}\label{eq:newpath}
	\{\tilde{\Wcal}^{\p{\lambda}},\lambda\in [0,1]\}
	\end{equation}
	denote the path in the parameter space, where each $\tilde{\Wcal}^{\p{\lambda}}$ equals $\Wcal^{\p{\lambda}}$ with the last layer weights re-scaled by $\tilde{c}^{\p{\lambda}}$, we have that \eqref{eq:newpath}  indeed defines a continuous path from the initialization point $\Wcal^{\p{0}}$ in the parameter space, along which the loss $L(P(\tilde{\Wcal}^{\p{\lambda}})$ is strictly monotonically decreasing. 
	
	All that remains now is to argue that from $\tilde{\Wcal}^{\p{1}}$, we have a strictly monotonically decreasing path to a point whose loss equals  $L(P(\Wcal^{\p{1}}))$. To see this, note that by definition of $\tilde{\Wcal}^{\p{1}}$ and $v^{\p{1}}$, we have $L(P(\tilde{\Wcal}^{\p{1}})) = v^{\p{1}} > L(P(\Wcal^{\p{1}}))$. Therefore,
	\[
	L(c\cdot P(\Wcal^{\p{1}}))
	\]
	is convex in $c$, equals $L(P(\tilde{\Wcal}^{\p{1}}))$ at $c=\tilde{c}^{\p{\lambda}}$, and equals the strictly smaller value $L(P(\Wcal^{\p{1}}))$ at $c=1$. Therefore, by re-scaling the last layer parameters of $\tilde{\Wcal}^{\p{1}}$ to match those of $\Wcal^{\p{1}}$, we are guaranteed to strictly and monotonically decrease the loss, until we get a loss equal to $L(P(\Wcal^{\p{1}}))$. Concatenating this with the continuous path $\tilde{\Wcal}^{\p{\lambda}},\lambda\in [0,1]$, the result follows.

\subsection{Proof of \propref{prop:crosssquare}}
	\label{subsec:proofcrosssquare}
	It is enough to verify that for both losses, proposition \ref{prop:highinit} holds with $r=1$.
	
	For the squared loss, if $P(\Wcal^{\p{0}})\neq \mathbf{0}$, then consider the first training example $(\bx_i,y_i)$ for which $P(\Wcal^{\p{0}})(\bx_i)\neq \mathbf{0}$. In that case, it is easily verified that $(c\cdot P(\Wcal^{\p{0}})(\bx_i)-y_i)^2$ is strictly convex in $c$ (for any $c$), and therefore $L(c\cdot P(\Wcal^{\p{0}}))=\frac{1}{m}\sum_{i=1}^{m}(c\cdot P(\Wcal^{\p{0}})(\bx_i)-y_i)^2$ is also strictly convex, as an average of convex functions where at least one of them is strictly convex. Therefore, strict convexity holds with probability $r=1$.
	
	For the cross-entropy loss, it is enough to consider the first training example on which the prediction vector $\bp$ of $P(\Wcal^{\p{\lambda}})$ is non-zero, and show strict convexity on that example with probability $1$. Since the loss on other examples are convex as well, we get overall strict convexity with probability $1$ as required. Specifically, we need to show strict convexity in $c$ of the function 
	\begin{equation}\label{eq:logsumexp}
	-\log\left(\frac{\exp(c\cdot p_{j_i})}{\sum_{j}\exp(c\cdot p_j)}\right) = \log\left(\sum_j \exp(c\cdot p_{j})\right)-c\cdot p_{j_i}.
	\end{equation}
	where $j_i$ is the correct class. To do so, consider the function $f(\bp)=\log(\sum_j \exp(p_j))$. A straightforward calculation reveals that its Hessian equals
	\[
	\nabla^2 f(\bp) = \text{diag}(\bq)-\bq\bq^\top~~\text{where}~~ \bq = \frac{1}{\sum_j \exp(p_j)}\left(\exp(p_1),\exp(p_2),\ldots,\exp(p_k)\right),
	\]
	so the second derivative of the function in \eqref{eq:logsumexp} w.r.t. $c$ at some value $c$ equals
	\begin{equation}\label{eq:secondderivative}
	\bp^\top \nabla^2 f(c\cdot \bp)\bp.
	\end{equation}
	We now argue that this is strictly positive, unless $\bp$ is a constant vector $p_1=p_2=\ldots=p_k$, in which case the function in \eqref{eq:logsumexp} is indeed strictly convex. To see this, note that the Hessian of $f$ is a rank-1 perturbation of the $k\times k$ positive definite matrix $\text{diag}(\bq)$, so its rank is at least $k-1$. Thus, there is only a $1$-dimensional subspace of vectors $\bv$, for which $\bv^\top (\text{diag}(\bq)-\bq\bq^\top)\bv=0$, which can be verified to be exactly the subspace of constant vectors. Thus, \eqref{eq:secondderivative} is positive unless $\bp$ is a constant vector. 
	
	To finish the proof for the cross-entropy loss, it remains to show that the probability that $\bp$ is non-constant (conditioned on $P(\Wcal^{\p{0}})\neq\mathbf{0})$ is $1$. To simplify the notation, let $NZ$ be the event that $P(\Wcal^{\p{0}})\neq\mathbf{0}$, let $P$ be the event that conditioned on $NZ$, $\bp$ (the first non-zero prediction vector over the training examples) is also non-constant. Also, let $V$ be the event that conditioned on $NZ$, then for the same training example as $\bp$, the input vector to the output neurons is non-zero. Then it holds that
	\begin{equation}\label{eq:VP}
		\pr{P|NZ} = \frac{\pr{P|V,NZ}\pr{V|NZ}}{\pr{V|P,NZ}}.
	\end{equation}
	$\pr{P|V,NZ}=1$, since the linear output neurons are initialized independently from a spherically-symmetric distribution supported on non-zero vectors, so given a non-zero input, the probability that some neurons will output different values is $1$. Also, $\pr{V|P,NZ}=\pr{V|NZ}=1$, since conditioned on $NZ$, $\bp\neq \mathbf{0}$ by definition, and since the output neurons compute a homogeneous linear mapping, the input to these neurons must also be non-zero. Plugging these observations back into \eqref{eq:VP}, we get that $\pr{P|NZ}=1$ as required.

\subsection{Proof of Proposition \ref{prop:highinit}}\label{subsec:prophighinitproof}
	We first prove that 
	\begin{equation}
	\label{eq:prnzero}
	\pr{P(\Wcal^{\p{0}})\neq \mathbf{0}}\geq 1-2^{-n_{h-1}}.
	\end{equation}
	To see this, consider any neuron in the $\p{h-1}^\text{th}$ layer, computing $\bx\mapsto \relu{\inner{\bw,\bx}+b}$. Since $(\bw,b)$ is drawn from a spherically symmetric distribution supported on non-zero vectors, it holds for any fixed $\bx$ that  $\pr{\inner{\bw,\bx}+b>0}=\pr{\inner{\bw,\bx}+b<0}=\frac{1}{2}$. Therefore, $\pr{\relu{\inner{\bw,\bx}+b}\neq0}=\frac{1}{2}$. Since the weights at each neuron are drawn independently, and there are $n_{h-1}$ neurons in the $(h-1)^{\text{th}}$ layer, it follows that a linear output neuron receives a non-zero input with probability at least $1-2^{-n_{h-1}}$. If this event occurs, then the output of the output neuron will be non-zero with probability $1$. Since this holds for any fixed network input, it holds in particular for (say) the first training example, in which case $P(\Wcal^{\p{0}})$ will be non-zero with such a probability. Letting $A$ be the event that $P(\Wcal^{\p{0}})\neq \mathbf{0}$, as well as $L(c\cdot P(\Wcal^{\p{0}}))$ being strictly convex in $c\in [-1,+1]$, we have by \eqref{eq:prnzero} and the assumption in the statement that $\pr{A}\geq r\left(1-2^{-n_{h-1}}\right)$
	
	Let $W$ be the realization of the random variable $P(\Wcal^{\p{0}})$. Since the output neurons are initialized from a spherically symmetric distribution, $\pr{W}=\pr{-W}$ for any $W$. Moreover, it is easy to verify that for any $W$, event $A$ occurs for $P(\Wcal^{\p{0}})=W$ if and only if it occurs for $P(\Wcal^{\p{0}})=-W$. Therefore,
	\[
	\pr{W|A} = \frac{\pr{W}\pr{A|W}}{\pr{A}} = \frac{\pr{-W}\pr{A|-W}}{\pr{A}} = \pr{-W|A}. 
	\]
	In other words, conditioned on $A$, for any realization $W$, we are equally likely to get $W$ or $-W$. Also, conditioned on $A$ (which implies strict convexity), $\max\{L(W),L(-W)\}\geq \frac{L(W)+L(-W)}{2} > L\left(\frac{W+(-W)}{2}\right)=L(\mathbf{0})$. Therefore, by symmetry, $\pr{L(P(\Wcal^{\p{0}}))>L(\mathbf{0})~|~A}\geq \frac{1}{2}$. As a result, $\pr{L(P(\Wcal^{\p{0}}))>L(\mathbf{0})}\geq \frac{1}{2}\pr{A}\geq \frac{r}{2}\left(1-2^{-n_{h-1}}\right)$.

\subsection{Proof of \thmref{thm:brittleness_to_overspecification}}
\label{app:brittleness_to_overspecification}
	Denote for all $ j\in\pcc n $,
	\[
		S_j^+=\set{\bx\in S:x_j>0},S_j^-=\set{\bx\in S:x_j<0},
	\]
	and observe the objective value on $ S_j^+ $ 	satisfies for all $ j\in\pcc d $,
	\begin{align*}
		L_{S_j^+}\p{W,\bv} &= \sum_{t:\bx_t\in 		S_j^+}\ell\p{\sum_{i=1}^n v_i\relu{\inner{\bw_i,\bx_t}},y_t} \\
		&= \sum_{t:\bx_t\in 		S_j^+}\ell\p{x_{t,j}\sum_{i=1}^n v_i\relu{w_{i,j}},y_t}.
	\end{align*}
	Similarly,
	\[
		L_{S_j^-}\p{W,\bv} = \sum_{t:\bx_t\in S_j^-}\ell\p{-x_{t,j}\sum_{i=1}^n v_i\relu{-w_{i,j}},y_t}.
	\]
	Since $ \ell $ is convex, $ L_{S_j^+},L_{S_j^-} $ 	are convex in $ \sum_{i=1}^n v_i \relu{w_{i,j}},\sum_{i=1}^n v_i \relu{-w_{i,j}}, $ respectively, so their minimal values are well defined. It is clear that the minimum achievable using a single neuron is lower bounded by the minimum achievable using two-layer nets, $ \alpha $, which in turn is lower bounded by the average of all minimal objective values over the various $ S_j^{\pm} $. It then suffices to show that we initialize from a basin achieving such value, which we denote as $ \beta\le\alpha $, with high probability. Moreover, since the objective value on $ S_j^{\pm} $ is independent for each $ S_j^{\pm} $, it is enough to minimize the objective on each $ S_j^{\pm} $ separately.
	
	Since our basins correspond to the partition of our search space to a fixed sign at each coordinate, we have that for the expression $ \sum_{i=1}^n v_i \relu{w_{i,j}} $ to take the optimal value $ p^* $ in our initialized basin, it suffices that $ \sign\p{w_{i,j}}=1 $ and $ \sign\p{v_i}=\sign\p{p^*} $ for some $ i\in\pcc n $. Using an analogous argument for $ S_j^- $ we have,
	\begin{itemize}
		\item
		The probability of this condition not to hold for a single neuron is at most $ \frac{3}{4} $.
		\item
		The probability of this condition not to hold
		for all neurons (since by Assumption \ref{assum:init_dist} all neurons are independent) is at most $ \p{\frac{3}{4}}^n $.
		\item
		By using the union bound, the probability that
		exists some $ S_j^{\pm} $ such that no neuron can obtain the minimal value over it is at most
		\[
		2d\p{\frac{3}{4}}^n.
		\]
	\end{itemize}
	We conclude that when initializing $ \p{W,\bv} $
	using a distribution satisfying Assumption \ref{assum:init_dist} then
	\[
		\pr{\bas\p{W,\bv}\le\beta\le\alpha} \ge 1-2d\p{\frac{3}{4}}^n.
	\]

\subsection{Proof of \thmref{thm:low_intrinsic_data_dim}}
	\label{app:low_intrinsic_data_dim_proof}
	We will need the following two auxiliary lemmas:
	
	\begin{lemma}
		\label{lem:Lipschitz_loss}
		$N_{n}\p{\bw_{1},\dots,\bw_{n},\bv}\p{\bx}$ is $\p{\abs{v_i}\cdot\norm{\bx}}$-Lipschitz in each $\bw_i$.
	\end{lemma}
	
	We leave this lemma without proof, and note that it is immediate from the definition of $ N_n $.
	
	The following lemma provides a lower bound on the probability that the $ i^\text{th} $ neuron is initialized from a region with a point of distance at most $ \delta $ from $ \bw^*_i $.
	\begin{lemma}
		\label{lem:good_init_lower_bound}
		Let $ \delta>0 $, let $ \p{W^*,\bv^*} $ satisfying $ \norm{\bw^*_i}_2=1, ~\forall i\in\pcc n $, and let $ \p{W,\bv} $ be a point on an origin-centered sphere chosen uniformly at random. Then $ \forall i\in\pcc n $
		\begin{align*}
			\pr{\exists \tilde{\bw}_i: \norm{\tilde{\bw}_i-\bw_i^*}_2\le\delta, ~\sign\p{\inner{\tilde{\bw},\bx_t}}=\sign\p{\inner{\bw,\bx_t}}~\forall t\in\pcc m} \\ \ge\frac{1}{\pi\p{\rank\p{X}-1}}\p{\delta\sqrt{1-\frac{\delta^2}{4}}}^{\text{rank}\p{X}-1}.
		\end{align*}
	\end{lemma}
	Before turning to prove the lemma, we first prove the following auxiliary claim.
	
	\begin{claim}
		\label{clm:spherical_init_lower_bound}
		Let $ \delta>0 $ and let $ \ba\in\mathbb{S}^{d-1}\subseteq\bbr^d $ be a point on the $d$-dimensional unit sphere. Let $\bb$ be a point chosen uniformly at random from $\mathbb{S}^{d-1}$. Then
		\[
			\pr{\norm{\ba-\bb}_2\le\delta}\ge\frac{1}{\pi\p{d-1}}\p{\delta\sqrt{1-\frac{\delta^2}{4}}}^{d-1}.
		\]
	\end{claim}
	
	This claim suffices for proving a weaker version of \thmref{thm:low_intrinsic_data_dim} where $ \rank\p{X} $ is replaced with $ d $. However, utilizing a simple observation on the structure of the basin partition allows us to prove \lemref{lem:good_init_lower_bound} which strengthens the result.
	
	\begin{proof}
		For a point $\ba\in\mathbb{S}^{d-1}$, let $\mathsf{\bar{S}}\p{\ba,\theta}\coloneqq\set{\bb\in\mathbb{S}^{d-1}:\inner{\ba,\bb}\ge\cos\theta}$
		be the closed hyperspherical cap of angle $\theta\in\pcc{0,\pi}$. Note that if $\ba,\bb\in\mathbb{S}^{d-1}$ form an angle of $\theta^{\prime}\in\pcc{0,\theta}$ (i.e. $\bb\in \mathsf{\bar{S}}\p{\ba,\theta}$) then they form an isosceles triangle with apex angle $\theta^{\prime}$ and equal sides	of length $1$, so the distance between $\ba$ and $\bb$
		satisfies 
		\begin{align*}
			\norm{\ba-\bb} & = 2\sin\p{\frac{\theta^{\prime}}{2}}\\
			& \le 2\sin\p{\frac{\theta}{2}}.
		\end{align*}
		Taking $\delta\coleq2\sin\p{\frac{\theta}{2}}$ we have that $\theta=2\arcsin\p{\frac{\delta}{2}}$, so in order for us to lower bound $\pr{\norm{\ba-\bb}_{2}\le\delta}$, we need to compute the surface area $\nu_{d-1}\p{\theta}$ of the hyperspherical cap of angle $\theta$ at point $\ba$ (independent of $ \ba $), and normalize this quantity by the area of the hypersphere $\omega_{d-1}$.
		
		The surface area of a hyperspherical cap of radius $\theta$ is given
		by the formula: (\citep{li2011concise})
		\begin{equation}
		\label{eq:hyper-spherical_cap_area}
			\nu_{d-1}\p{\theta}=\omega_{d-2}\int_0^{\theta}\p{\sin^{d-2}\xi d\xi},
		\end{equation}
		where $\omega_{d-1}$ denotes the surface area of $\mathbb{S}^{d-1}$.		
		Consider the function $f\p{\theta}=\int_0^{\theta}\p{\sin^{d-2}\xi d\xi}-\frac{1}{d-1}\sin^{d-1}\theta$.
		It is monotonically increasing in $\pcc{0,\pi}$ since
		\begin{align*}
			f^{\prime}\p{\theta} &= \frac{\partial}{\partial\theta}\p{\int_{0}^{\theta}\p{\sin^{d-2}\xi d\xi}-\frac{1}{d-1}\sin^{d-1}\theta}\\
			& = \sin^{d-2}\theta-\sin^{d-2}\theta\cdot\cos\theta\\
			& = \sin^{d-2}\theta\cdot\p{1-\cos\theta}\\
			& \ge 0,
		\end{align*}
		where the last inequality holds for all $ \theta\in\pcc{0,\pi} $. Since $f\p 0=0$ we have $\forall\theta\in\pcc{0,\pi}$ that $\int_{0}^{\theta}\p{\sin^{d-2}\xi d\xi}\ge\frac{1}{d-1}\sin^{d-1}\theta$.
		
		We compute
		\begin{align*}
			\pr{A_{d}} &=  \frac{\omega_{d-2}}{\omega_{d-1}}\int_{0}^{\theta}\p{\sin^{d-2}\xi d\xi}\\
			&\ge  \frac{\omega_{d-2}}{\omega_{d-1}}\cdot\frac{\sin^{d-1}\theta}{d-1}\\
			&= \frac{\omega_{d-2}}{\omega_{d-1}}\cdot\frac{\sin^{d-1}\p{2\arcsin\p{\frac{\delta}{2}}}}{d-1}.
		\end{align*}
		Using the identities $\sin\p{\arcsin x}=x$, $\cos\p{\arcsin x}=\sqrt{1-x^2}$ and $\sin2x=2\sin x\cdot\cos x$, we have
		\[
			\sin^{d-1}\p{2\arcsin\p{\frac{\delta}{2}}}=\p{\delta\sqrt{1-\frac{\delta^{2}}{4}}}^{d-1}.
		\]
		Finally, $\frac{\omega_{d-2}}{\omega_{d-1}}$ can be shown to be monotonically increasing for all $d\ge2$, so $\frac{\omega_{d-2}}{\omega_{d-1}}\ge\frac{\omega_{0}}{\omega_{1}}=\frac{1}{\pi}$,	thus yielding
		\[
			\pr{A_{d}}\ge\frac{1}{\pi\p{d-1}}\p{\delta\sqrt{1-\frac{\delta^{2}}{4}}}^{d-1},
		\]
		which concludes the proof of the claim.
	\end{proof}
	
	We now turn to prove \lemref{lem:good_init_lower_bound}.
	\begin{proof}[Proof (of \lemref{lem:good_init_lower_bound})]
		Let $U=\text{span}\p{\bx_1,\dots,\bx_m}$, and define $T\p{\bx}\coleq\frac{\bu}{\norm{\bu}_2}$ where $\bx=\bu+\bu^{\perp}$ for $\bu\in U,\bu^{\perp}\in U^{\perp}$.
		First, we observe that for any initialization of $\p{W,\bv}$ , that $\p{W,\bv}$ and $ \p{T\p{W},\bv} $ where $T\p W\coleq\p{T\p{\bw_1},\dots,T\p{\bw_n}}$ both belong to the same basin, since $\forall i\in\pcc n,\forall t\in\pcc m$
		\begin{align*}
			\inner{\bx_t,\bw_i}
			&= \inner{\bx_t,\norm{\bw}_2\cdot T\p{\bw_i}+\bw_i^\perp}\\
			&= \inner{\bx_t,\norm{\bw}_2\cdot T\p{\bw_i}}+\inner{\bx_t,\bw_i^\perp}\\
			&= \inner{\bx_t,\norm{\bw}_2\cdot T\p{\bw_i}}\\
			&= \norm{\bw}_2\cdot\inner{\bx_t,T\p{\bw_i}},
		\end{align*}
		\[
			\implies\sign\p{\inner{\bx_t,\bw_i}} = \sign\p{\inner{\bx_t,T\p{\bw_i}}}.
		\]
		Thus both $W,T\p W$ belong to the same basin, achieving the same minimal value. Since any rotation $\Theta$ under which $U^\perp$ is invariant commutes with $T$, we have for any measurable set $A\subseteq U$
		\[
			\sigma_{\rank\p{X}-1}\p A=\sigma_{d-1}\p{\Theta T^{-1}\p A}=\sigma_{d-1}\p{T^{-1}\p{\Theta A}},
		\]
		where $ \sigma\p{k} $ denotes the $k$-dimensional Lebesgue measure. So initializing uniformly on an origin-centered sphere of dimension $d$ is equivalent to initializing uniformly on an origin-centered sphere of dimension $\rank\p{X}$ in the sense of the region we initialize from.	We complete the proof by invoking Claim \ref{clm:spherical_init_lower_bound} with respect to a $ \p{\rank\p{X}} $-dimensional sphere.
	\end{proof}
	
	We are now ready to prove \thmref{thm:low_intrinsic_data_dim}.
	\begin{proof}[Proof (of \thmref{thm:low_intrinsic_data_dim})]
		We first argue that since our initialization distribution satisfies Assumption \ref{assum:init_dist}, we may rescale each neuron once initialized to the unit sphere. This is justified since a linear rescaling of the weight of each neuron does not change the basin we initialized from, so the basin value remains the same. For this reason, we assume without loss of generality the distribution where each neuron is distributed uniformly and independently on the unit sphere. Define
		\[
			p_\epsilon=\frac{1}{2\pi\p{\rank\p{X}-1}}\p{\frac{\sqrt{\epsilon}}{nB}\cdot\sqrt{1-\frac{\epsilon}{4n^2 B^2}}}^{\rank\p{X}-1} = \Omega\p{\p{\frac{\sqrt \epsilon}{nB}}^{\text{rank}\p{X}}}.
		\]
		Using the positive homogeneity of the ReLU, we can rescale each $ v_i^* $ to satisfy $ \abs{v_i^*}=1 ~\forall i\in\pcc n $, and rescale $ \bw_i^* $ accordingly, so we may also assume $ \abs{v_i^*}=1,\norm{\bw_i^*}\le B~\forall i\in \pcc n $. From \lemref{lem:good_init_lower_bound} we have
		\begin{align*}
			& \pr{\exists \tilde{\bw}_i: \norm{\norm{\bw_i^*}\cdot\tilde{\bw}_i-\bw_i^*}_2\le\frac{\sqrt{\epsilon}}{n}, ~\sign\p{\inner{\tilde{\bw},\bx_t}}=\sign\p{\inner{\bw,\bx_t}}~\forall t\in\pcc m} \\
			=& \pr{\exists \tilde{\bw}_i: \norm{\tilde{\bw}_i-\frac{\bw_i^*}{\norm{\bw_i^*}}}_2\le\frac{\sqrt{\epsilon}}{n\norm{\bw_i^*}}, ~\sign\p{\inner{\tilde{\bw},\bx_t}}=\sign\p{\inner{\bw,\bx_t}}~\forall t\in\pcc m} \\
			\ge& \pr{\exists \tilde{\bw}_i: \norm{\tilde{\bw}_i-\frac{\bw_i^*}{\norm{\bw_i^*}}}_2\le\frac{\sqrt{\epsilon}}{nB}, ~\sign\p{\inner{\tilde{\bw},\bx_t}}=\sign\p{\inner{\bw,\bx_t}}~\forall t\in\pcc m} \\
			=& 2p_\epsilon,
		\end{align*}
		and also
		\[
		\pr{\sign\p{v_i}=\sign\p{v_i^*}}=\frac{1}{2}.
		\]
		Since the two events are independent, we have that both occur w.p.
		at least $p_\epsilon$. Also, this event is independent for each neuron,
		so we have w.p. at least $p_\epsilon$ for each neuron to initialize
		`close` enough to $\p{\bw_i^*,v_i^*}$. In this sense,
		we can lower bound the number of good initializations from below using
		$Z\sim B\p{N,p_\epsilon}$, where $B\p{N,p}$ is the binomial
		distribution. By using Chernoff`s bound we bound the tail of $Z$ as follows
		\begin{align*}
			& F\p{n;c\left\lceil\frac{n}{p_\epsilon}\right\rceil,p_\epsilon}
			\\ 
			\le & \exp\p{-\frac{1}{2p_\epsilon}\frac{\p{c\ceil{\frac{n}{p_\epsilon}}p_\epsilon-n}^2}{c\ceil{\frac{n}{p_\epsilon}}}}\\
			\le & \exp\p{-\frac{1}{2}\frac{\p{cn-n}^2}{cn}}\\
			\le & \exp\p{-\frac{1}{4}cn}.
		\end{align*}
		Thus with probability $\ge 1-\exp\p{-\frac{1}{4}cn}$,
		we have $n$ neurons initialized in a basin containing a point $ \tilde{W} \in \bbr^{n\times d} $ of distance at most $ \frac{\sqrt{\epsilon}}{n} $ from an optimal weight $ \bw_i^* $ for each $ i\in\pcc n $.
		
		Let $ i_1,\dots,i_n $ be the indices of the well initialized neurons, and let
		\[
			\tilde{W}_{i}=\p{\bw_1^*,\dots,\bw_i^*,\norm{\bw_{i+1}^*}\tilde{\bw}_{i+1},\dots,\norm{\bw_n^*}\tilde{\bw}_n}, ~i=0,\dots,n.
		\]
		We compute the value of the basin corresponding to these neurons as follows:
		\begin{align*}
			\bas\p{\bw_{i_1},\dots,\bw_{i_n},v_{i_1},\dots,v_{i_n}}
			&\le L\p{\tilde{W},\bv^*} \\
			&= \frac{1}{m}\sum_{t=1}^{m}\p{N_n\p{\tilde{W}_0,\bv^*}\p{\bx_t}-y_t}^2\\
			&= \frac{1}{m}\sum_{t=1}^m\abs{\sum_{i=1}^n\p{N_n\p{\tilde{W}_{i-1},\bv^*}\p{\bx_t}-N_n\p{\tilde{W}_i,\bv^*}\p{\bx_t}}}^2\\
			&\le \frac{1}{m}\sum_{t=1}^m\p{\sum_{i=1}^n\abs{N_n\p{\tilde{W}_{i-1},\bv^*}\p{\bx_t}-N_n\p{\tilde{W}_i,\bv^*}\p{\bx_t}}}^2\\
			&\le \frac{1}{m}\sum_{t=1}^m\abs{\abs{v^*_i}\norm{\bx_t}\p{\sum_{i=1}^n\norm{\tilde{W}_{i-1}-\tilde{W}_i}}}^2\\
			&\le \p{\sum_{i=1}^{n}\norm{\tilde{W}_{i-1}-\tilde{W}_{i}}}^{2}\\
			&= \p{\sum_{i=1}^n\norm{\norm{\bw_i^*}\cdot\tilde{\bw}_i-\bw_i^*}}^2\\
			&\le \p{\sum_{i=1}^{n}\frac{\sqrt{\epsilon}}{n}}^{2}\\
			&= \epsilon,
		\end{align*}
		
		where the second inequality in the triangle inequality and the third inequality is from \lemref{lem:Lipschitz_loss}. We now finish the proof by invoking \lemref{lem:key_lemma} to conclude
		\[
			\pr{\bas\p{W,\bv} \le \bas\p{\bw_{i_1},\dots,\bw_{i_n},v_{i_1},\dots,v_{i_n}} \le \epsilon} \ge 1-e^{-\frac{1}{4}cn}.
		\]
	\end{proof}
	
\subsection{Proof of \thmref{thm:Rank_m_data}}
	\label{app:Rank_m_data}
	Denote the initialization point as $W=\p{\bw_1,\dots,\bw_n,v_1,\dots,v_n}$, and define $ \p{W^\prime,\bv^\prime} $ with $ \bv^\prime=\p{\sign\p{v_1},\dots,\sign\p{v_n}} $,  $\bw_i^{\prime}=\sum_{t^{\prime}=1}^{m}a_{i,t^{\prime}}\bx_{t^{\prime}}$ where $a_{i,t^{\prime}}\in\bbr$ are to be determined later. Let $$ \p{\bar{y}_1,\dots,\bar{y}_m}=\argmin{\p{\bar{y}_1,\dots,\bar{y}_m}\in\bbr^m}\frac{1}{m}\sum_{t=1}^m \ell\p{\bar{y}_t,y_t}, $$ we want to show that for well chosen values of $a_{i,t^{\prime}}$, $\p{W^\prime,\bv^\prime}$ belongs to the same basin as $\p{W,\bv}$, and achieves the desired prediction $ \p{\bar{y}_1,\dots,\bar{y}_m} $ over a certain subset of $\p{\bx_1,\dots,\bx_m}$, while achieving a prediction of $0$ over the rest of the sample instances, effectively predicting the subset without affecting the prediction over the rest of the sample. By combining enough neurons in this manner, we are able to obtain the minimal objective value over the data. Namely, an objective value of $\alpha$.
	Define the vector  $\by_i^{\prime}=\p{y_{i,1}^{\prime},\ldots, y_{i,m}^{\prime}}^\top$, where 
	\[
		y_{i,t}^{\prime}=
		\begin{cases}
			\abs{\bar{y}_t} & \inner{\bw_i,\bx_t}>0,\; v_i\cdot \bar{y}_t\ge0,\;\forall j<i\:\inner{\bw_j,\bx_t}\le0\wedge v_j\cdot \bar{y}_t<0\\
			0 & \text{otherwise}
		\end{cases}.
	\]
	and choose $\ba_i=\p{a_{i,1},\ldots,a_{i,m}}^\top$ such that the equality
	\[
		XX^{\top}\ba_i=\by_i^{\prime}
	\]
	holds. 
	
	We first stress that by our assumption,
	\[
		XX^{\top}=\p{\begin{matrix}\inner{\bx_1,\bx_1} & \dots & \inner{\bx_1,\bx_m}\\
		\vdots & \ddots & \vdots\\
		\inner{\bx_m,\bx_1} & \dots & \inner{\bx_m,\bx_m}
		\end{matrix}}\in \bbr^{m\times m}
	\]
	is of rank $m$, and therefore $\ba_i$ exists and is well-defined.
	
	Assuming that for any $t\in\pcc m$ there exists some neuron $i$	such that $\inner{\bw_i,\bx_t}>0,\; v_i\cdot \bar{y}_t\ge0$ (We will later analyze the probability of this actually happening), we compute the prediction of our network with weights $W^{\prime}=\p{\bw_1^{\prime},\dots,\bw_n^{\prime}}$ on $\bx_t$:
	\begin{align*}
		N_n\p{W^{\prime},\bv^\prime}\p{\bx_t} &= \sum_{i=1}^n v_i^\prime\relu{\inner{\bw_i^{\prime},\bx_t}}\\
		&= \sum_{i=1}^n v_i\relu{\inner{\sum_{t^{\prime}=1}^m a_{i,t^{\prime}}\bx_{t^{\prime}},\bx_t}}\\
		&= \sum_{i=1}^n v_i\relu{\sum_{t^{\prime}=1}^m a_{i,t^{\prime}}\inner{\bx_{t^{\prime}},\bx_t}}\\
		&= \sum_{i=1}^n v_i\relu{y_{i,t}^{\prime}}\\
		&= \sum_{i=1}^n v_i\relu{\abs{\bar{y}_t}\cdot\one_{\inner{\bw_i,\bx_t}>0,\; v_i\cdot \bar{y}_t\ge 0,\;\forall j<i\:\inner{\bw_j,\bx_t}\le 0\wedge v_j\cdot \bar{y}_t<0}}\\
		&= \sum_{i=1}^n \bar{y}_t \cdot\one_{\inner{\bw_i,\bx_t}>0,\; v_{i}\cdot \bar{y}_t\ge0,\;\forall j<i\:\inner{\bw_j,\bx_t}\le 0\wedge v_j\cdot \bar{y}_t<0}\\
		&= \bar{y}_t.
	\end{align*}
	Where the last equality comes from our assumption that there exists some neuron $i$ s.t. $\inner{\bw_i,\bx_t}>0,\; v_i\cdot \bar{y}_t\ge 0$, and from the definition of $y_{i,t}^{\prime}$ which asserts that at most a single neuron will predict $\bx_t$. Thus, we have
	\[
		\forall t\in\pcc m\; N_n\p{W^{\prime},\bv}\p{\bx_t}=\bar{y}_t
	\]
	\[
		\implies L_S\p{W^{\prime},\bv}=\alpha.
	\]
	To put this result in different words, if $\bx_t$ is positive on the hyperplane induced by $\bw_i$ and if $v_i$ has the same sign as $\bar{y}_t$, then $\bw_i^{\prime}$ predicts $\bx_t$ correctly, given that $\bx_t$ was not previously predicted by a neuron $\bw_j^{\prime}$ where	$j<i$.
	
	We now assert that $ \p{W,\bv} $ and $ \p{W^\prime,\bv^\prime} $ indeed belong to the same basin with respect to $S$. For $ \bv,\bv^\prime $ this is clear by definition, and note that for $\bw_i,\bw_i^{\prime}$ we require that $\sign\p{\inner{\bw_i,\bx_t}}\cdot\sign\p{\inner{\bw_i^{\prime},\bx_t}}\ge 0,~\forall i\in\pcc n,t\in\pcc m$.
	Thus we compute:\\
	If $\inner{\bw_i,\bx_t}>0,\; v_i\cdot \bar{y}_t\ge 0,\;\forall j<i\:\inner{\bw_j,\bx_t}\le 0\wedge v_j\cdot \bar{y}_t<0$	all hold, then we have
	\[
		\sign\p{\inner{\bw_i,\bx_t}}=1,
	\]
	and
	\begin{align*}
		\sign\p{\inner{\bw_i^{\prime},\bx_t}} &= \sign\p{\inner{\bw_i^{\prime},\bx_t}}\\
		&= \sign\p{\inner{\sum_{t^{\prime}=1}^m a_{i,t^{\prime}}\bx_{t^{\prime}},\bx_t}}\\
		&= \sign\p{\sum_{t^{\prime}=1}^m a_{i,t^{\prime}}\inner{\bx_{t^{\prime}},\bx_t}}\\
		&= \sign\p{y_{i,t}^{\prime}}\\
		&= \sign\p{\abs{\bar{y}_t}}\\
		&\ge 0.
	\end{align*}
	Otherwise, we have $ \sign\p{\inner{\bw_i,\bx_t}}\le 0 $ and
	\[
		\sign\p{\inner{\bw_i^{\prime},\bx_t}} = \sign\p{y_{i,t}^{\prime}}= 0.
	\]
	Finally, we define the event $A_i^t\coleq\inner{\bw_i,\bx_t}>0,\; v_i\cdot \bar{y}_t\ge0$,
	i.e. the $i^{\text{th}}$ neuron is able to predict $\bx_t$ correctly. Since $v_i,\bw_i$ are independent, and since $ \bw_i $ is drawn from a spherically symmetric distribution for all $ i\in\pcc n $, we have
	\[
		\pr{A_i^t} = \pr{\inner{\bw_i,\bx_t}>0}\cdot\pr{v_i\cdot, \bar{y}_t\ge 0} \ge \frac{1}{4}
	\]
	\[
		\implies\pr{\overline{A_i^t}}\le\frac{3}{4}.
	\]
	Since the neurons are independent, and since $ \p{\sign\p{v_1},\dots,\sign\p{v_n}} $ is uniformly distributed on the Boolean cube, we have
	\[
		\pr{\bigcap_{i=1}^n\overline{A_i^t}}\le\p{\frac{3}{4}}^n.
	\]
	Using the union bound on $\bigcap_{i=1}^n\overline{A_i^t}$ for $t=1,\dots, m$ we get
	\[
		\pr{\exists t\text{ s.t. no neuron predicts }\bx_t} \le m\p{\frac{3}{4}}^n.
	\]
	Thus, the probability of initializing from a basin achieving a global minimum with value $\alpha$ is at least
	\[
		1-m\p{\frac{3}{4}}^{n}.
	\]

\subsection{Proof of \thmref{thm:Clustered_data}}
\label{app:Clustered_data_proof}
The idea behind the proof is comprised of two parts. The first is that by predicting the clusters' centers well, we are able to obtain a good objective value over the data. The second is that the basin partition of the clustered data is similar to the basin partition of the clusters' centers. So by approximating a good solution for the clusters' centers, we are able to reach a good objective value.

Recall that by Definition \ref{def:Basin_partition}, we partition the parameter space $ \bbr^{n\times d} $ of the first layer into sets where $ \sign\p{\inner{\bw_i,\bx_t}} $ is fixed for all $ i\in\pcc n,t\in\pcc m $. This restricts the possible weight vector of each neuron in the first layer to a subset of $ \bbr^d $. Referring to these subsets as \emph{regions}, we observe that their structure varies slightly when changing $ \delta $ from $ 0 $ (where a cluster contains a single point) to a small positive quantity, where the new regions introduced by the clusters are referred to as \emph{noisy regions} (see \figref{fig:Noisy_regions}).

\begin{figure}
	\centering
	\includegraphics[scale=0.32]{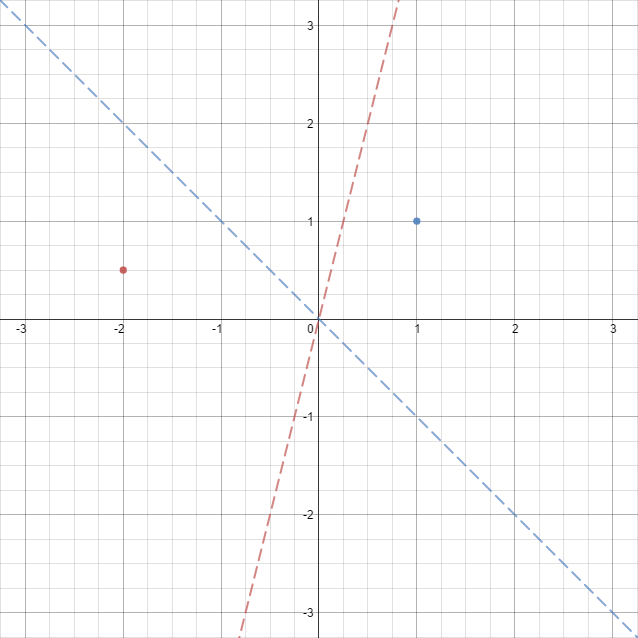}~~\includegraphics[scale=0.32]{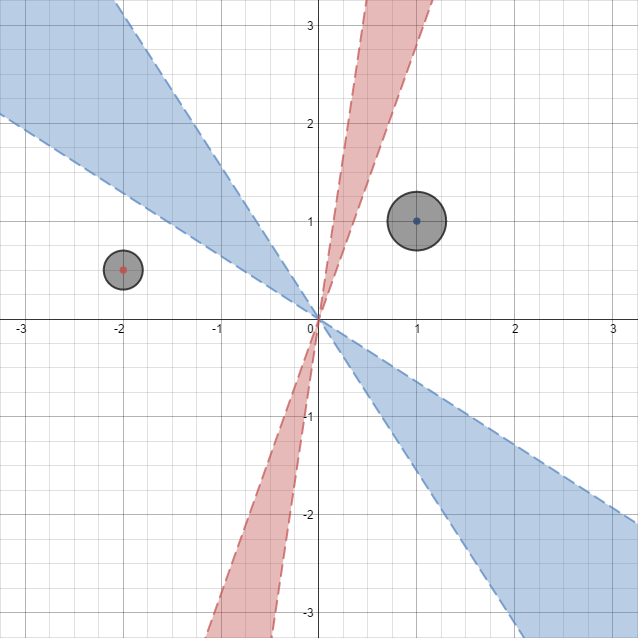}
	\caption{\footnotesize The partition of $ \bbr^2 $ into regions by the instances $\bc_1=\p{1,1},\bc_2=\p{-2,0.5}$, and the corresponding partition by clustered instances with centers $ \bc_1,\bc_2 $. The noisy regions are depicted by the light blue and light red.}
	\label{fig:Noisy_regions}
\par\end{figure}

To approximate a good solution for the clusters' centers, we need to initialize from a basin where such an approximation exists. Note that if $ \delta=0 $, then the result will hold as a corollary of \thmref{thm:Rank_m_data}. Alternatively, if $ \delta $ is small enough, then we would expect such an approximation to exist in the basins comprised of non-noisy regions, as these vary slightly when $ \delta $ is small. Therefore, we would like to assert that we initialize from these basins to guarantee the existence of a good solution.

Before delving into the proof of \thmref{thm:Clustered_data}, we first prove two auxiliary lemmas (\lemref{lem:non-noisy_init_lower_bound} and \lemref{lem:cluster_centers_surrogate}). The following lemma provides an upper bound on initializing a single neuron from a noisy region, for distributions satisfying Assumption \ref{assum:init_dist}.
\begin{lemma}
	\label{lem:non-noisy_init_lower_bound}
	Define the set of noisy regions with respect to the $ j^\text{th} $ cluster,
	\[
		A_j=\set{\bx:\norm{\bx}_2=1,\;\exists \by:\norm{\bc_j-\by}_2\le\delta_j,~\inner{\bx,\by}=0}.
	\]
	Then under the assumptions in \thmref{thm:Clustered_data}, its complement with respect to the $d$-dimensional unit sphere $A_j^c=\mathbb{S}^{d-1}\backslash A_j$ satisfies
	\[
		\frac{\sigma_{d-1}\p{A_j^c}}{\omega_{d-1}}\ge1-\frac{1}{4d}.
	\]
	Where $ \sigma_{d-1} $ is the $ \p{d-1} $-dimensional Lebesgue measure, and $ \omega_{d-1} $ is the surface area of the $d$-dimensional unit sphere.
\end{lemma}

To prove the lemma we will need two auxiliary claims.

\begin{claim}
	\label{clm:Set_containing_spherical_caps}
	Let $\mathsf{S}\p{\ba,\theta}\coloneqq\set{\bb\in\mathbb{S}^{d-1}:\inner{\ba,\bb}>\cos\theta}$ denote the open hyperspherical cap of spherical radius $\theta$ and center $\bx$. Then
	\[
	\mathsf{S}\p{\bc_j,\frac{\pi}{2}-2\arcsin\frac{\delta_j}{2\norm{\bc_j}}}\dot{\cup}\mathsf{S}\p{-\bc_j,\frac{\pi}{2}-2\arcsin\frac{\delta_j}{2\norm{\bc_j}}}\subseteq A_j^c.
	\]
\end{claim}

\begin{proof}
	Clearly, the two open hyperspherical caps are disjoint, as they are of spherical radius $\le\frac{\pi}{2}$ and the two originate in two diametrically opposite points. Assume $\bx\in\mathsf{S}\p{\bc_j,\frac{\pi}{2}-2\arcsin\frac{\delta_j}{2\norm{\bc_j}}}$, then the projection of $ \set{\bz:\norm{\bc_j-\bz}_2\le\delta_j} $ onto $\mathbb{S}^{d-1}$, denoted $P_j$, is a hyperspherical cap of spherical radius $\theta\coleq 2\arcsin\frac{\delta_j}{2\norm{\bc_j}}$. Since the dot product is a bi-linear operation, it suffices to show that $\forall\by\in P_j\;\inner{\tilde{\bx},\by}\neq 0$, where $\tilde{\bx}\in\mathbb{S}^{d-1}$ is the projection of $\bx$ onto $\mathbb{S}^{d-1}$.
	
	Let $\by\in P_j$, using the fact that $\mathsf{s}:\mathbb{S}^{d-1}\times \mathbb{S}^{d-1}\to \bbr_+$, the spherical distance function defined by $ \mathsf{s}\p{\ba,\bb}\coleq\arccos\p{\inner{\ba,\bb}} $, satisfies the triangle inequality we have
	\begin{align*}
		\mathsf{s}\p{\tilde{\bx},\by} &\le \mathsf{s}\p{\tilde{\bx},\bc_j}+\mathsf{s}\p{\bc_j,\by}\\
		&< \frac{\pi}{2}-\theta+\theta\\
		&= \frac{\pi}{2},
	\end{align*}
	\[
		\implies\inner{\tilde{\bx},\by}\neq 0.
	\]
	Where the same argument works for $\bx\in\mathsf{S}\p{-\bc_j,\frac{\pi}{2}-2\arcsin\frac{\delta_j}{2\norm{\bc_j}}}$
	and $-P_j$.
\end{proof}

Moving on to our next auxiliary claim.
\begin{claim}
	\label{clm:Integral_lower_bound}
	$ \forall\theta\ge 0 $ we have
	\[
		\int_0^{\frac{\pi}{2}-\theta}\sin^{d-2}\xi d\xi\ge\frac{\omega_{d-1}}{2\omega_{d-2}}-\theta.
	\]
\end{claim}

\begin{proof}
	Consider the function $f\p{\theta}=\p{\int_0^{\frac{\pi}{2}-\theta}\sin^{d-2}\xi d\xi}-\p{\frac{\omega_{d-1}}{2\omega_{d-2}}-\theta}$,
	it is monotonically increasing in $\pco{0,\infty}$ since
	\begin{align*}
		f^{\prime}\p{\theta} &=
		\frac{\partial}{\partial\theta}\p{\p{\int_0^{\frac{\pi}{2}-\theta}\sin^{d-2}\xi d\xi}-\p{\frac{\omega_{d-1}}{2\omega_{d-2}}-\theta}}\\
		&= -\sin^{d-2}\p{\frac{\pi}{2}-\theta}+1\\
		&\ge 0.
	\end{align*}
	And since $f\p 0=0$ we have $\forall\theta\in\pco{0,\infty}$ that $\int_0^{\frac{\pi}{2}-\theta}\sin^{d-2}\xi d\xi\ge\frac{\omega_{d-1}}{2\omega_{d-2}}-\theta$.
\end{proof}

We now turn to prove \lemref{lem:non-noisy_init_lower_bound}.

\begin{proof}[Proof (of \lemref{lem:non-noisy_init_lower_bound})]
	Using Claims \ref{clm:Set_containing_spherical_caps} and
	\ref{clm:Integral_lower_bound}, 			\eqref{eq:hyper-spherical_cap_area} and the fact that $\forall d\ge2\;\frac{\omega_{d-1}}{\omega_d}\le\sqrt{\frac{d}{2\pi}}$
	(\citep{leopardi2007distributing}, Lemma~2.3.20), we have the following:
	\begin{align*}
		\frac{\sigma_{d-1}\p{A_j^c}}{\omega_{d-1}} &\ge \frac{\sigma_{d-1}\p{\mathsf{S}\p{\bc_j,\frac{\pi}{2}-\theta}}+\sigma_{d-1}\p{\mathsf{S}\p{-\bc_j,\frac{\pi}{2}-\theta}}}{\omega_{d-1}}\\
		&= \frac{2\nu_{d-1}\p{\frac{\pi}{2}-\theta}}{\omega_{d-1}}\\
		&= 2\frac{\omega_{d-2}}{\omega_{d-1}}\int_0^{\frac{\pi}{2}-\theta}\sin^{d-2}\xi d\xi\\
		&\ge 2\frac{\omega_{d-2}}{\omega_{d-1}}\p{\frac{\omega_{d-1}}{2\omega_{d-2}}-\theta}\\
		&= 1-\frac{2\omega_{d-2}\theta}{\omega_{d-1}}\\
		&\ge 1-4\sqrt{\frac{d}{2\pi}}\cdot\arcsin\frac{\delta_j}{2\norm{\bc_j}}\\
		&\ge 1-4\sqrt{\frac{d}{2\pi}}\cdot\arcsin\p{\sin\p{\frac{\sqrt{2\pi}}{16d\sqrt{d}}}}\\
		&= 1-\frac{1}{4d}.
	\end{align*}
\end{proof}

Now that we can bound the probability of initializing from a noisy region $ A_j,j\in\pcc k $, we turn to show that with high probability, a solution with $ \mathcal{O}\p{\delta^2} $ value can be found. Let $ C $ be the matrix with rows $ \bc_1,\dots,\bc_k $, then by \thmref{thm:Rank_m_data} we know that with high probability there exists some $ \p{\tilde{W},\tilde{\bv}} $ which achieves a value of $ 0 $ on the dataset $ \set{\bc_j,\hat{y}_j}_{j=1}^k $, and since the cluster target values are $ \gamma $-Lipschitz, this $ \p{\tilde{W},\tilde{\bv}} $ will also perform well on $ S $. Unfortunately, we cannot guarantee that $ \p{\tilde{W},\tilde{\bv}} $ resides in the basin we initialized from, as this guarantee can only be given on the basin partition where $ \delta=0 $. Instead, we take a surrogate $ \p{W^\prime,\bv^\prime} $ which approximates $ \p{\tilde{W},\tilde{\bv}} $ well, and then show that the value it achieves is also of magnitude $ \mathcal{O}\p{\delta^2} $. More formally, we have the following lemma.

\begin{lemma}
	\label{lem:cluster_centers_surrogate}
	Let $ C $ be a matrix with rows $ \bc_1,\dots,\bc_k $, satisfying $ \rank\p{C}=k $. Let $ \p{W,\bv}\in B_S^{A,\bb} $ satisfy $ \forall j\in\pcc k,\; \exists i\in\pcc n:\; \bw_i\notin \cup_r A_r,\inner{\bw_i,\bc_j}>0,v_i\cdot\hat{y}_j\ge0$. Then exist $ \p{\tilde{W},\tilde{\bv}} $ and $ \p{W^\prime,\bv^\prime} $ where the following holds:
	\begin{enumerate}
		\item 
		$ \p{\tilde{W},\tilde{\bv}} $ predicts $ y_t $ well:
		\[
			\abs{N_n\p{\tilde{W},\tilde{\bv}}-y_t}\le \delta\p{ n\frac{\sigma_{\max}\p{C^\top}}{\sigma_{\min}^2\p{C^\top}}\norm{\hat{\by}}_2+2\gamma}.
		\]
		\item
		$ \p{W^\prime,\bv^\prime}\in B_S^{A,\bb} $ approximates $ \p{\tilde{W},\tilde{\bv}} $ well:
		\[
			\abs{N_n\p{W^\prime,\bv^\prime}-N_n\p{\tilde{W},\tilde{\bv}}} \le nB\cdot\frac{\delta}{c}\frac{\sigma_{\max}\p{C^\top}}{\sigma_{\min}^2\p{C^\top}}\norm{\hat{\by}}_2.
		\]
		
	\end{enumerate}
\end{lemma}

Before proving the lemma, we state and prove the following two auxiliary claims.

\begin{claim}
	\label{clm:w_tilde_norm_bound}
	Let $\tilde{\bw}_i\coleq\sum_{j=1}^k a_{i,j}\bc_j=C^{\top}\ba_i$, where $ \ba_i $ satisfies the equality $ CC^{\top}\ba_i=\by_i^{\prime} $ as in \appref{app:Rank_m_data}. Then for all $ i\in\pcc n $,
	\[
		\norm{\tilde{\bw}_i}_2 \le \frac{\sigma_{\max}\p{C^\top}}{\sigma_{\min}^2\p{C^\top}}\norm{\hat{\by}}_2.
	\]
\end{claim}

\begin{proof}
	We derive a bound on $\norm{\tilde{\bw}_i}_2$ as follows:
	\begin{align*}
		CC^{\top}\ba_i &= \by_i^{\prime},\\
		\implies\ba_i &= \p{CC^{\top}}^{-1}\by_i^{\prime},\\
		\implies C^{\top}\ba_i &= C^{\top}\p{CC^{\top}}^{-1}\by_i^{\prime},\\
		\implies\norm{\tilde{\bw}_i}_2 &= \norm{C^{\top}\p{CC^{\top}}^{-1}\by_i^{\prime}}_2\\
		&\le \norm{C^{\top}}_{\text{op}}\norm{\p{CC^{\top}}^{-1}}_{\text{op}}\norm{\by_i^{\prime}}_2\\
		&= \sigma_{\max}\p{C^\top}\cdot\frac{1}{\sigma_{\min}^2\p{C^\top} }\norm{\by_i^{\prime}}_2\\
		&\le \frac{\sigma_{\max}\p{C^\top}}{\sigma_{\min}^2\p{C^\top}}\norm{\hat{\by}}_2.
	\end{align*}
\end{proof}

\begin{claim}
	\label{clm:Lipschiz_in_x} $N_n\p{\bw_1,\dots,\bw_n,\bv}\p{\bx}$ is $\p{\sum_{i=1}^n\abs{v_i}\cdot\norm{\bw_i}}$-Lipschitz in $\bx$.
\end{claim}
The proof of this claim follows the same idea behind \lemref{lem:Lipschitz_loss}, and is therefore omitted.

We are now ready to prove \lemref{lem:cluster_centers_surrogate}.
\begin{proof}[Proof (of \lemref{lem:cluster_centers_surrogate})]
	We first define $ \p{\tilde{W},\tilde{\bv}} $ as the point satisfying $ E_{S^\prime}\p{\tilde{W},\tilde{\bv}}=0 $, as demonstrated in \appref{app:Rank_m_data}. Defining $ \p{W^\prime,\bv^\prime} $, we let $ \bv^\prime=\tilde{\bv} \in \set{-1,+1}^n $. If $ \p{\tilde{\bw}_i,\bw_i^\prime} $ both belong to the same region with respect to $ S $, then take $ \bw_i^\prime = \tilde{\bw}_i $. Otherwise, we approximate $ \tilde{\bw}_i $ in the $ \norm{\cdot}_2 $ sense, by taking $ \bw_i^\prime $ in the region we initialized from which is closest to $ \tilde{\bw}_i $.
	
	\begin{enumerate}
		\item
		We compute using Claims \ref{clm:w_tilde_norm_bound} and \ref{clm:Lipschiz_in_x},
		\begin{align*}
		& \abs{N_n\p{\tilde{W},\tilde{\bv}}\p{\bx_t}-y_t} \\
		=&\abs{N_n\p{\tilde{W},\tilde{\bv}}\p{\bx_t}-N_n\p{\tilde{W},\tilde{\bv}}\p{\mathsf{c}\p{\bx_t}}+N_n\p{\tilde{W},\tilde{\bv}}\p{\mathsf{c}\p{\bx_t}}-y_t}\\
		\le& \abs{N_n\p{\tilde{W},\tilde{\bv}}\p{\bx_t}-N_n\p{\tilde{W},\tilde{\bv}}\p{\mathsf{c}\p{\bx_t}}}+\abs{N_n\p{\tilde{W},\tilde{\bv}}\p{\mathsf{c}\p{\bx_t}}-y_t}\\
		\le& \sum_{i=1}^n \norm{\tilde{\bw}_i}\cdot\norm{\bx_t-\mathsf{c}\p{\bx_t}}+\abs{\hat{y}_t-y_t}\\
		\le & \delta \p{n\frac{\sigma_{\max}\p{C^\top}}{\sigma_{\min}^2\p{C^\top}}\norm{\hat{\mathbf{y}}}_{2}+2\gamma}.
		\end{align*}
		where the last inequality comes from $y_t,\hat{y}_t$ being the target values of points belonging to a ball of diameter at most $2\delta$ and the target values being $\gamma$-Lipschitz.
		\item
		Note that by definition we have $ \p{W^\prime,\bv^\prime}\in B_S^{A,\bb} $. Denote the origin as $O$. In the worst case, $ \tilde{\bw} $ is on the line connecting $ O $ and $ \bc_i $, so assume this is the case. Denote the point at which the line connecting $O$ and $\tilde{\bw}_i$ is tangent to the $ i^\text{th} $ cluster by $H_i$, then the vertices $ O,\bw_i^\prime,\tilde{\bw}_i $ and $O,H_i,\bc_i$ form similar triangles, and we have
		\[
			\norm{\bw_i^{\prime}-\tilde{\bw}_i}_2 = \frac{\delta_i}{\norm{\bc_i}_2}\norm{\tilde{\bw}_i}_2
			\le \frac{\delta}{c}\norm{\tilde{\bw}_i}_2.
		\]
		Now, using Claim \ref{clm:w_tilde_norm_bound} and \lemref{lem:Lipschitz_loss},
		\begin{align*}
			&\abs{N_n\p{W^\prime,\bv^\prime}\p{\bx_t}-N_n\p{\tilde{W},\tilde{\bv}}\p{\bx_t}} \\
			=& \abs{N_n\p{W^\prime,\bv^\prime}\p{\bx_t}-N_n\p{\tilde{W},\bv^\prime}\p{\bx_t}} \\
			\le& \sum_{i=1}^n\abs{v_i}\cdot\norm{\bx_t}_2\cdot\norm{\bw_i^{\prime}-\tilde{\bw}_i}_2\\
			\le& \sum_{i=1}^n B\cdot\frac{\delta}{c}\norm{\tilde{\bw}_i}_2\\
			\le& \sum_{i=1}^{n}B\cdot\frac{\delta}{c}\frac{\sigma_{\max}\p{C^\top}}{\sigma_{\min}^2\p{C^\top}}\norm{\hat{\by}}_2\\
			=& nB\cdot\frac{\delta}{c}\frac{\sigma_{\max}\p{C^\top}}{\sigma_{\min}^2\p{C^\top}}\norm{\hat{\by}}_2.
		\end{align*}
	\end{enumerate}
\end{proof}

Equipped with the above lemmas, we are now ready to prove \thmref{thm:Clustered_data}.
\begin{proof}[Proof (of \thmref{thm:Clustered_data})]
	Using \lemref{lem:non-noisy_init_lower_bound}, we have for all $ j\in\pcc k $
	\[
		\frac{\sigma_{d-1}\p{A_j^c}}{\omega_{d-1}}\ge1-\frac{1}{4d}.
	\]
	Applying the union bound to the $k\le d$ events where we initialize from $A_j$, we have that we don't initialize a single neuron from a noisy region w.p. at least $\frac{3}{4}$. For a given $j\in\pcc k$, using the union bound again, the probability of initializing from a non-noisy region in which any internal point $\bw\in\bbr^d$	satisfies $\inner{\bw,\bc_j}>0$ is at least $\frac{1}{4}$, and finally, since $v_i$ has the correct sign w.p. $\frac{1}{2}$ and is independent of where we initialize $\bw_i$ from, we are unable to predict $\bc_j$ w.p. at most $\frac{7}{8}$. Using the union bound once more in the same manner as we did in \appref{app:Rank_m_data} gives that we initialize ``properly'' w.p. at least
	\[
		1-k\p{\frac{7}{8}}^{n}\ge1-d\p{\frac{7}{8}}^{n}.
	\]
	We stress that by using \lemref{lem:key_lemma}, for the purpose of analyzing the objective value, we can ignore initializations made from noisy regions, as we may just consider the neurons that were properly initialized. By our assumption that the clusters' centers are in general position, namely that the matrix $ C $ with rows $ \bc_1,\dots,\bc_k $ satisfies $ \sigma_{\min}\p{C^\top} > 0$, we have that it is in particular of rank $ k $, and the conditions in \lemref{lem:cluster_centers_surrogate} are met, so we compute
	\begin{align*}
		L_S\p{W^\prime,\bv^\prime} &= \frac{1}{m}\sum_{t=1}^m\p{N_n\p{W^\prime,\bv^\prime}\p{\bx_t}-\hat{y}_t}^2\\
		&= \frac{1}{m}\sum_{t=1}^{m}\abs{N_n\p{W^\prime,\bv^\prime}\p{\bx_t}-N_n\p{\tilde{W},\tilde{\bv}}\p{\bx_t}+N_n\p{\tilde{W},\tilde{\bv}}\p{\bx_t}-\hat{y}_t}^2\\
		&\le \frac{1}{m}\sum_{t=1}^m\p{\abs{N_n\p{W^\prime,\bv^\prime}\p{\bx_t}-N_n\p{\tilde{W},\tilde{\bv}}\p{\bx_t}}+\abs{N_n\p{\tilde{W},\tilde{\bv}}\p{\bx_t}-\hat{y}_t}}^2\\
		&\le \frac{1}{m}\sum_{t=1}^m\p{nB\cdot\frac{\delta}{c}\frac{\sigma_{\max}\p{C^\top}}{\sigma_{\min}^2\p{C^\top}}\norm{\hat{\by}}_2+\delta n\frac{\sigma_{\max}\p{C^\top}}{\sigma_{\min}^2\p{C^\top}}\norm{\hat{\by}}_2+2\gamma\delta}^2\\
		&= \delta^2\p{\p{1+\frac{B}{c}}n\frac{\sigma_{\max}\p{C^\top}}{\sigma_{\min}^2\p{C^\top}}\norm{\hat{\by}}_2+2\gamma}^2.
	\end{align*}
	Thus we conclude that when $ \p{W,\bv} $ is initialized using a  distribution satisfying Assumption \ref{assum:init_dist}, we have
	\[
		\pr{\bas\p{W,\bv}\le \delta^2\p{\p{1+\frac{B}{c}}n\frac{\sigma_{\max}\p{C^\top}}{\sigma_{\min}^2\p{C^\top}}\norm{\hat{\by}}_2+2\gamma}^2}\ge 1-d\p{\frac{7}{8}}^n.
	\]
\end{proof}

\section{Poor Basin Structure for Single Neurons}
\label{app:eps_realizability_single_neurons}
In this appendix, we prove a hardness result for initializing ReLU single neuron nets with convex losses from a basin (as will shortly be defined for the single neuron architecture context) with a good basin value, and then provide an explicit construction for the squared loss.

For a single neuron, our objective function with respect to a ReLU activation and convex loss function $ \ell $ is
\[
	L_S\p{\bw}=\frac{1}{m}\sum_{t=1}^m\ell\p{\relu{\inner{\bw,\bx_t}},y_t},
\]
corresponding to the parameter space $ \bbr^d $.
As done in \secref{sec:Neural_nets_depth_two} for two-layer networks, we can partition the parameter space according to the signs of $\inner{\bw,\bx_t}$ on each training instance $\bx_t$. Each region in this partition corresponds to an intersection of halfspaces, in which our objective $L_S(\bw)$ can easily be shown to be convex. Thus, each such region corresponds to basin (as defined in Definition \ref{def:Basin}), and we can consider the probability of initializing in a basin with low minimal value. 

\subsection{Exponentially Many Poor Local Minima}
	\label{app:exp_many_local_minima}
	Building on the work of \citep{auer1996exponentially}, we provide a construction of a dataset which results in exponentially many poor local minima in the dimension. Moreover, we provide in subsection \appref{app:hardness_example} an explicit construction for the squared loss. 
	The results extend those of \citep{auer1996exponentially} by showing that they hold for a single neuron with the ReLU activation function (for which the technical conditions assumed in \citep{auer1996exponentially} do not apply).
	
	From an optimization point of view, having exponentially many local minima is not necessarily problematic as many of which may obtain good objective values. However, following our initialization scheme throughout this work, we modify the result obtained in \citep{auer1996exponentially} to satisfy that when the weight vector of the neuron is initialized from a distribution satisfying Assumption \ref{assum:init_dist}, then the distribution of the minimal value in the basin we initialize from is strongly concentrated around a sub-optimal value as the dimension increases. More formally, we have the following Theorem.
	
	\begin{theorem}
		\label{thm:Hardness_result_single_neuron}
		Consider a ReLU single neuron neural net, with a convex and symmetric loss function $ \ell $ satisfying $\ell\p{a,b}=0 $ if and only if $ a=b $. Then for all $ \epsilon>0 $ there exists a sample $ S $ such that $ L_S\p{\bw^*}=\epsilon $ for some $ \bw^*\in\bbr^d $, and a constant $c\in\bbr$ which depends only on $\ell$, such that the objective value over the sample $L_S$ contains $2^d$ strict local minima, and
		\[
			\pr{\bas\p{\bw}\le\frac{c}{4}}\le e^{-\frac{d}{16}}.
		\]
		Where $ \bw $ is initialized according to Assumption \ref{assum:init_dist}.
	\end{theorem}
	
	In other words, we have exponentially many local minima, where the probability of initializing from a sub-optimal basin converges exponentially fast (in the dimension) to $1$, yet there exists a solution which obtains a value of $\epsilon$. 
	
	\begin{proof}
		Let $ L_0=\ell\p{0,1} $. Since $ \relu{0}=0\neq1=\relu{1} $ we have $ L_0>0 $. We are interested in a construction where $ \epsilon $ is small enough, therefore assume $ \epsilon<\frac{L_0}{2} $. By the continuity of $ \ell $ as a convex function, we can find $ \delta\in\p{0,1} $ small enough such that $ \ell\p{0,\delta}=2\epsilon $.
		
		Consider the sample
		\[
			S=\set{\p{x_1=\delta,y_1=\delta},\p{x_2=-1,y_2=1}}.
		\]
		We compute
		\[
			\ell\p{\relu{wx_1},y_1}=
			\begin{cases}
				2\epsilon & w\in\poc{-\infty,0}\\
				0 & w=1
			\end{cases},
		\]
		\[
			\ell\p{\relu{wx_2},y_2}=
			\begin{cases}
				0 & w=-1\\
				L_0 & w\in\pco{0,\infty}
			\end{cases}.
		\]
		Therefore the objective value over $ S $, $ L_S\p{w}=\frac{1}{2}\p{\ell\p{\relu{wx_1},y_1}+\ell\p{\relu{wx_2},y_2}} $ satisfies
		\[
			L_S\p{w}=
			\begin{cases}
				\epsilon & w=-1 \\
				\frac{L_0}{2} & w=1 \\
				> \frac{L_0}{2} & w\in\pco{0,1}\cup\p{1,\infty}
			\end{cases}.
		\]
		But since $ \ell $ is convex, we have that $ L_S $ is convex in $ \poc{-\infty,0} $ and in $ \pco{0,\infty} $, so $ L_S $ has exactly two local minima, one is $ L_S\p{-1}=\epsilon $ and the other is $ L_S\p{1}=\frac{L_0}{2} $.
		
		We now extend our sample to be $d$-dimensional in a similar manner as did the authors in \citep{auer1996exponentially} as follows: For $t=1,2$ and $j\in\pcc d$, we use the mapping $x_t\p{j}\mapsto\p{0,\dots,0,x_t,0,\dots,0}$	where the non-zero coordinate is the $j^{\text{th}}$ coordinate. It is straightforward to show that the partial derivative $\frac{\partial}{\partial w_j}L_S$ is $0$ for $x_t\p{k}$ with $j\neq k$, so the geometry of the surface of the objective function $L_S$ is independent for each coordinate. Now, every Cartesian product of local minima in the one-dimensional setting form a $d$-dimensional local minimum. Since we have exactly two local minima, a good and another bad one in each coordinate, this combines into $2^d$ local minima, where each minimum's value would be the average of the one-dimensional minima forming it. Note that the combination of all good minima forms the global minimum with value $\epsilon$. Following standard convention, we say that the data in this case is \emph{$\epsilon$-realizable} using a single neuron architecture. We stress that an important property of this initialization scheme is that the signs of the coordinates of the initialization point is uniformly distributed on the Boolean cube, as it implies that on each coordinate, independently, we have a probability $0.5$ of reaching a bad basin, hence the number of bad basins we initialize from is distributed according to a Binomial distribution $B\p{d,0.5}$. Letting $c=\frac{L_0}{2}$, we have from Chernoff`s bound that
		\[
			\pr{\bas\p{\bw}\le\frac{c}{4}}\le e^{-\frac{d}{16}},
		\]
		which concludes the proof of the theorem.
	\end{proof}
	
\subsection{An Explicit Construction With the Squared Loss}

		\begin{figure}[t]
			\centering
			\includegraphics[scale=0.2]{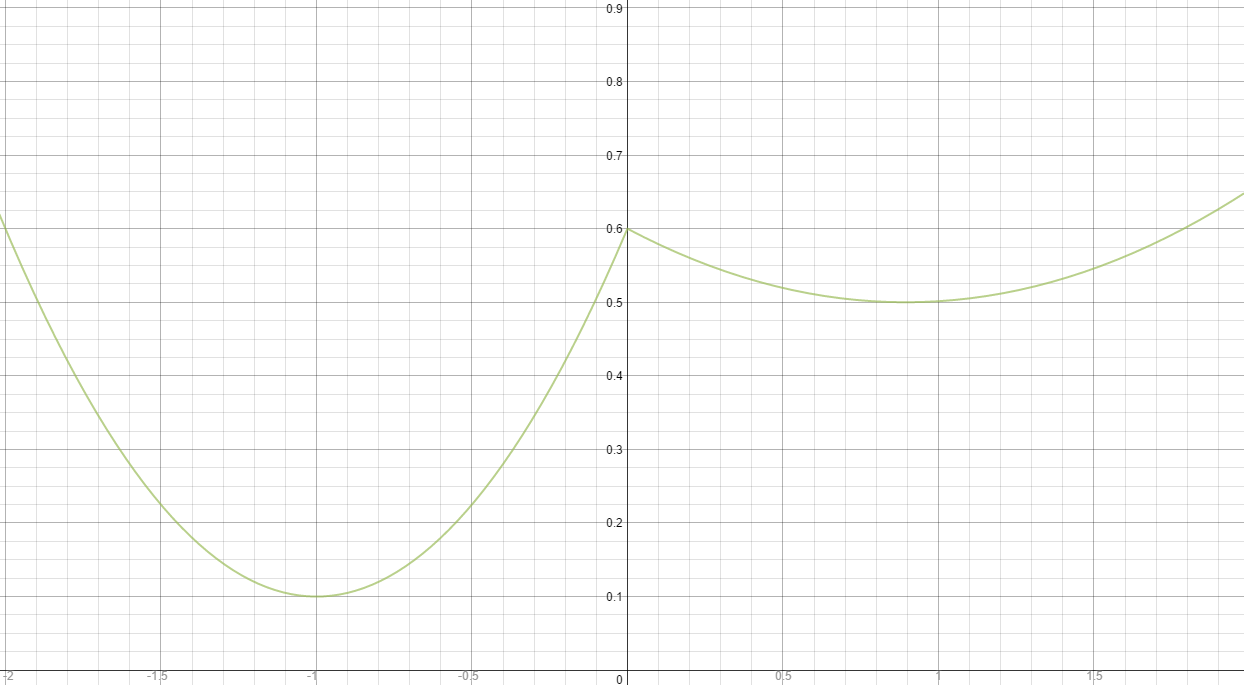}
			\caption{\footnotesize Plot of $L_S\p w$ for $\epsilon=0.1$.}
			\label{fig:ReLU_error}
			\par\end{figure}

		\begin{figure}[t]
			\centering
			\includegraphics[scale=0.5]{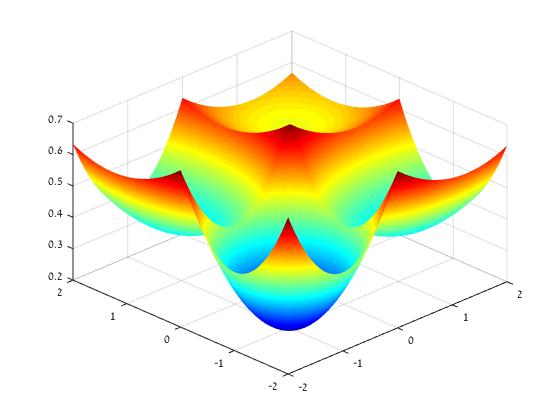}
			\caption{\footnotesize Plot of $L_S\p w$ after extending the sample to	2 dimensions. The surface contains one optimal minimum, another bad minimum and 2 average valued minima.}
			\label{fig:2D_ReLU_hardness_example}
			\par\end{figure}

	\label{app:hardness_example}
	We illustrate a specific construction of \thmref{thm:Hardness_result_single_neuron}, for ReLU paired with the squared loss.
	
	Define
	\[
		\ell\p{y,y^{\prime}}=\p{y-y^{\prime}}^2.
	\]
	Given $\epsilon>0$, consider the following sample:
	\[
		S=\set{\p{x_1=\frac{1}{2},y_1=\sqrt{2\epsilon}},\p{x_2=-1,y_2=1}}.
	\]
	Define for $i=1,2$
	\[
		\ell_i\p w=\p{\relu{wx_i}-y_i}^2,
	\]
	and denote
	\[
		L_S\p w=\frac{1}{2}\p{\ell_1\p w+\ell_2\p w}.
	\]
	Note that
	\[
		L_S\p{-1}=\epsilon,
	\]
	\[
		L_S\p{2\sqrt{2\epsilon}}=\frac{1}{2}
	\]
	are both local minima, and thus $S$ is $\epsilon$-realizable. As evident in \figref{fig:ReLU_error} and \figref{fig:2D_ReLU_hardness_example}, if we are using a distribution corresponding to Assumption \ref{assum:init_dist}, then we have a $50\%$ chance to initialize from the bad basin. \\
	Extending the sample into a $d$-dimensional one as we did in \thmref{thm:Hardness_result_single_neuron}, we have an $\epsilon$-realizable dataset $S$ with $2^d$ local minima. Furthermore, we have that 
	\[
		\pr{\bas\p{\bw}\le\frac{1}{8}}\le e^{-\frac{d}{16}}.
	\]
		
\end{document}